\documentclass[sigconf, nonacm]{acmart}
\newcommand\vldbdoi{XX.XX/XXX.XX}
\newcommand\vldbpages{XXX-XXX}
\newcommand\vldbvolume{14}
\newcommand\vldbissue{1}
\newcommand\vldbyear{2020}
\newcommand\vldbauthors{\authors}
\newcommand\vldbtitle{\shorttitle} 
\newcommand\vldbavailabilityurl{URL_TO_YOUR_ARTIFACTS}
\newcommand\vldbpagestyle{plain}

\usepackage{balance}  
\usepackage{epsfig}
\usepackage{multirow}
\usepackage{graphicx}
\usepackage{subfigure}
\usepackage{algorithm}
\usepackage{algorithmicx}
\usepackage{algpseudocode}
\usepackage{xspace}
\usepackage{enumerate}
\usepackage{array}
\newtheorem{theorem}{Theorem}
\newtheorem{definition}{Definition}
\newtheorem{lemma}{Lemma}

\def\eg {\emph{e.g}.} 
\def\ie {\emph{i.e}.}

\def\etal{\emph{et al}.}

\begin{document}


\title{Fast Approximate Nearest Neighbor Search With The Navigating Spreading-out Graph}



%
%
%
%

\author{Cong Fu}
\affiliation{%
  \institution{Zhejiang University}
  \city{Hangzhou}
  \country{China}
}
\email{fc731097343@gmail.com}

\author{Chao Xiang}
\affiliation{%
  \institution{Zhejiang University}
  \city{Hangzhou}
  \country{China}
}
\email{chaoxiang@zju.edu.cn}

\author{Changxu Wang}
\affiliation{%
  \institution{Zhejiang University}
  \city{Hangzhou}
  \country{China}
}
\email{changxu.mail@gmail.com}

\author{Deng Cai}
\affiliation{%
  \institution{Zhejiang University}
  \city{Hangzhou}
  \country{China}
}
\email{dengcai@gmail.com}


\begin{abstract}
Approximate nearest neighbor search (ANNS) is a fundamental problem in databases and data mining. A scalable ANNS algorithm should be both memory efficient and fast. Some early graph-based approaches have shown attractive theoretical guarantees on search time complexity, but they all suffer from the problem of high indexing time complexity. Recently, some graph-based methods have been proposed to reduce indexing complexity by approximating the traditional graphs; these methods have achieved revolutionary performance on million-scale datasets. Yet, they still can not scale to billion-node databases. In this paper, to further improve the search-efficiency and scalability of graph-based methods, we start by introducing four aspects: (1) ensuring the connectivity of the graph;  (2) lowering the average out-degree of the graph for fast traversal;  (3) shortening the search path; and (4) reducing the index size. Then, we propose a novel graph structure called Monotonic Relative Neighborhood Graph (MRNG) which guarantees very low search complexity (close to logarithmic time). To further lower the indexing complexity and make it practical for billion-node ANNS problems, we propose a novel graph structure named Navigating Spreading-out Graph (NSG) by approximating the MRNG. The NSG takes the four aspects into account simultaneously. Extensive experiments show that NSG outperforms all the existing algorithms significantly. In addition, NSG shows superior performance in the E-commercial search scenario of Taobao (Alibaba Group) and has been integrated into their search engine at billion-node scale.
\end{abstract}

\maketitle

\pagestyle{\vldbpagestyle}
\begingroup\small\noindent\raggedright\textbf{PVLDB Reference Format:}\\
\vldbauthors. \vldbtitle. PVLDB, \vldbvolume(\vldbissue): \vldbpages, \vldbyear.\\
\href{https://doi.org/\vldbdoi}{doi:\vldbdoi}
\endgroup
\begingroup
\renewcommand\thefootnote{}\footnote{\noindent
This work is licensed under the Creative Commons BY-NC-ND 4.0 International License. Visit \url{https://creativecommons.org/licenses/by-nc-nd/4.0/} to view a copy of this license. For any use beyond those covered by this license, obtain permission by emailing \href{mailto:info@vldb.org}{info@vldb.org}. Copyright is held by the owner/author(s). Publication rights licensed to the VLDB Endowment. \\
\raggedright Proceedings of the VLDB Endowment, Vol. \vldbvolume, No. \vldbissue\ %
ISSN 2150-8097. \\
\href{https://doi.org/\vldbdoi}{doi:\vldbdoi} \\
}\addtocounter{footnote}{-1}\endgroup

\ifdefempty{\vldbavailabilityurl}{}{
\vspace{.3cm}
\begingroup\small\noindent\raggedright\textbf{PVLDB Artifact Availability:}\\
The source code, data, and/or other artifacts have been made available at \url{\vldbavailabilityurl}.
\endgroup
}

\section{Introduction}
Approximate nearest neighbor search (ANNS) has been a hot topic over decades and provides fundamental support for many applications in data mining, databases, and information retrieval~\cite{HuangFZFN15,AroraSK018,teodoro2014approximate,zheng2016lazylsh,de2002efficient,chen2005robust}. For sparse discrete data (like documents), the nearest neighbor search can be carried out efficiently on advanced index structures (e.g., inverted index \cite{Manning2008Introduction}). For dense continuous vectors, various solutions have been proposed such as tree-structure based approaches \cite{Bentley1975Multidimensional, Silpaanan2008Optimised, Fukunaga1975A, jagadish2005idistance, beckmann1990r, AroraSK018}, hashing-based approaches \cite{Gionis1999Similarity, Weiss2008Spectral, gao2014dsh, liu2014sk, HuangFZFN15}, quantization-based approaches \cite{jegou2011product, weber1998quantitative, ge2013optimized, andre2015cache}, and graph-based approaches \cite{Hajebi2011Fast, arya1993approximate, wu2014fast, malkov2014approximate}. Among them, graph-based methods have shown great potential recently. There are some experimental results showing that the graph-based methods perform much better than some popular algorithms from other types in the commonly used Euclidean Space \cite{Jin2014Fast,malkov2014approximate,MalkovYHNSW16,Ben2016Fanng,AroraSK018,CongEfanna2016}. The reason may be that these methods cannot express the neighbor relationship as well as the graph-based methods and they tend to check much more points in neighbor-subspaces than the graph-based methods to reach the same accuracy\cite{weber1998quantitative}. Thus, their search time complexity involves large factors exponential in the dimension and leads to inferior performance \cite{har2012approximate}.

Nearest neighbor search via graphs has been studied for decades \cite{jaromczyk1992relative, dearholt1988monotonic, arya1993approximate}. Given a set of points $S$ in the $d$-dimensional Euclidean space $E^d$, a graph $G$ is defined as a set of edges connecting these points (nodes). The edge $pq$ defines a neighbor-relationship between node $p$ and $q$. Various constraints are proposed on the edges to make the graphs suitable for ANNS problem. These graphs are now referred to as the \textit{Proximity Graphs} \cite{jaromczyk1992relative}. Some proximity graphs like Delaunay Graphs (or Delaunay Triangulation) \cite{aurenhammer1991voronoi} and Monotonic Search Networks (MSNET) \cite{dearholt1988monotonic} ensure that from any node $p$ to another node $q$, there exists a path on which the intermediate nodes are closer and closer to $q$ \cite{dearholt1988monotonic}. However, the computational complexity needed to find such a path is not given. Other works like Randomized Neighborhood Graphs \cite{arya1993approximate} guarantee polylogarithmic search time complexity. Empirically, the average length of greedy-routing paths grows polylogarithmically with the data size on the Navigable Small-World Networks (NSWN) \cite{kleinberg2000navigation, boguna2009navigability}. However, the time complexity of building these graphs is very high (at least \textit{O}($n^2$)), which is impractical for massive problems. 

Some recent graph-based methods try to address this problem by designing approximations for the graphs. For example, GNNS \cite{Hajebi2011Fast}, IEH \cite{Jin2014Fast}, and Efanna \cite{CongEfanna2016} are based on the $k$NN graph, which is an approximation of the Delaunay Graph. NSW \cite{malkov2014approximate} approximates the NSWN, FANNG \cite{Ben2016Fanng} approximates the Relative Neighborhood Graphs (RNG) \cite{toussaint1980relative}, and Hierarchical NSW (HNSW) \cite{MalkovYHNSW16} is proposed to take advantage of properties of the Delaunay Graph, the NSWN, and the RNG. Moreover, a hierarchical structure is used in HNSW to enable multi-scale hopping on different layers of the graph. 

These approximations are mainly based on intuition and generally lack rigorous theoretical support. In our experimental study, we find that they are still not powerful enough for billion-node applications, which are in great demand today. To further improve the search-efficiency and scalability of graph-based methods, we start with how ANNS is performed on a graph. Despite the diversity of graph indices, almost all graph-based methods \cite{dearholt1988monotonic, arya1993approximate, Ben2016Fanng, malkov2014approximate, Hajebi2011Fast, Jin2014Fast} share the same greedy best-first search algorithm (given in Algorithm \ref{search_alg}), we refer to it as the \textit{search-on-graph} algorithm below.

\begin{algorithm}[t]\small
	\caption{Search-on-Graph($G$, $\textbf{p}$, $\textbf{q}$, $l$)}
	\label{search_alg}
	\begin{algorithmic}[1]
		\Require graph $G$, start node $\textbf{p}$, query point $\textbf{q}$, candidate pool size $l$
		\Ensure $k$ nearest neighbors of $\textbf{q}$
		\State $i$=0, candidate pool $S = \emptyset$
		\State $S$.add($\textbf{p}$)
		\While{$i < l$}
		\State $i=$the index of the first unchecked node in $S$
		\State mark $\textbf{p}_\textbf{i}$ as checked
		\ForAll {neighbor $\textbf{n}$ of $\textbf{p}_\textbf{i}$ in $G$}
		\State $S$.add($\textbf{n}$)
		\EndFor
		\State sort $S$ in ascending order of the distance to $\textbf{q}$
		\State If $S$.size() $>l$, $S$.resize($l$) 
		\EndWhile
		\State return the first $k$ nodes in $S$
	\end{algorithmic}
\end{algorithm}

Algorithm \ref{search_alg} tries to reach the query-node by the following greedy process. For a given query $q$, we are required to retrieve its nearest neighbors from the dataset. Algorithm \ref{search_alg} tries to reach the query point with the following greedy process. Given a starting node $p$, we follow the out-edges to reach $p$'s neighbors, and compare them with $q$ to choose one to proceed. The choosing principle is to minimize the distance to $q$, and the new iteration starts from the chosen node. We can see that the key to improve graph-based search is to shorten the search path formed by the algorithm and reduce the out-degree of the graph (\ie, reduce the number of choices of each node). Intuitively, to improve graph-based search we need to: (1) Ensure the connectivity of the graph to make sure the query (or the nearest neighbors of the query) is (are) reachable; (2) Lower the average out-degree of the graph and (3) shorten the search path to lower the search time complexity; (4) Reduce the index size (memory usage) to improve scalability. Methods such as IEH \cite{Jin2014Fast}, Efanna \cite{CongEfanna2016}, and HNSW \cite{MalkovYHNSW16}, use hashing, randomized KD-trees and multi-layer graphs to accelerate the search. However, these may result in huge memory usage  for massive databases. We aim to reduce the index size and preserve the search-efficiency at the same time. 

In this paper, we propose a new graph structure, named as Monotonic Relative Neighborhood Graph (MRNG), which guarantees a low average search time complexity (very close to logarithmic time complexity). To further reduce the indexing complexity, we propose the Navigating Spreading-out Graph (NSG), which is a good approximation of MRNG, inherits low search complexity and takes the four aspects into account. It is worthwhile to highlight our contributions as follows.
\begin{enumerate}
\item We first present comprehensive theoretical analysis on the attractive ANNS properties of a graph family called MSNET. Based on this, we propose a novel graph, MRNG, which ensures a close-logarithmic search complexity in expectation. 
\item To further improve the efficiency and scalability of graph-based ANNS methods, we consider four aspects of the graph: ensuring connectivity, lowering the average out-degree, shortening the search path, and reducing the index size. Motivated by these, we design a close approximation of the MRNG, called Navigating Spreading-out Graph (NSG), to address the four aspects simultaneously. The indexing complexity is reduced significantly compared to the MRNG and is practical for massive problems. Extensive experiments show that our approach outperforms the state-of-the-art methods in search performance with the smallest memory usage among graph-based methods.
\item The NSG algorithm is also tested on the E-commercial search scenario of Taobao (Alibaba Group). The algorithm has been integrated into their search engine for billion-node search.
\end{enumerate}

\section{PRELIMINARIES}

We use $E^d$ to denote the Euclidean space under the $l_2$ norm. The closeness of any two points $p,q$ is defined as the $l_2$ distance, $\delta(p,q)$, between them.
\subsection{Problem Setting}
Various applications in information retrieval and database management of high-dimensional data can be abstracted as the nearest neighbor search problem in high-dimensional space.  The Nearest Neighbor Search (NNS) problem is defined as follows \cite{Gionis1999Similarity}:
\begin{definition}[\textbf{Nearest Neighbor Search}]
Given a finite point set $S$ of $n$ points in space $E^d$, preprocess $S$ to efficiently return a point $p \in S$ which is closest to a given query point $q$.
\end{definition}
This naturally generalizes to the \textbf{$K$ Nearest Neighbor Search} when we require the algorithm to return $K$ points ($K>1$) which are the closest to the query point. The approximate version of the nearest neighbor search problem (ANNS) can be defined as follows \cite{Gionis1999Similarity}:
\begin{definition}[\textbf{$\epsilon$-Nearest Neighbor Search}]
Given a finite point set $S$ of $n$ points in space $E^d$, preprocess $S$ to efficiently answer queries that return a point $p$ in $S$ such that $\delta(p,q) \le (1+\epsilon)\delta(r,q)$, where $r$ is the nearest neighbor of $q$ in $S$. 
\end{definition}
Similarly, this problem can generalize to the \textbf{Approximate $K$ Nearest Neighbor Search (AKNNS)} when we require the algorithm to return $K$ points ($K>1$) such that $\forall{i = 1,...,K},\delta(p_i,q) \le (1+\epsilon)\delta(r,q)$. Due to the intrinsic difficulty of exact nearest neighbor search, most researchers turn to AKNNS. The main motivation is to trade a little loss in accuracy for much shorter search time. 

For the convenience of modeling and evaluation, we usually do not calculate the exact value of $\epsilon$. Instead we use another indicator to show the degree of the approximation: \textit{precision}. Suppose the point set returned by an AKNNS algorithm of a given query $q$ is $R'$ and the correct $k$ nearest neighbor set of $q$ is $R$, then the \textit{precision} (accuracy) is defined as below \cite{CongEfanna2016}. 
\begin{equation}
precision(R') = \frac{|R' \cap R|}{|R'|} = \frac{|R' \cap R|}{K}.
\end{equation}
A higher \textit{precision} corresponds to a smaller $\epsilon$, thus, a higher degree of approximation. In this paper, we use the \textit{precision} as the evaluation metric.

\subsection{Non-Graph-Based ANNS Methods}
Non-graph-based ANNS methods include tree-based methods, hashing-based methods, and quantization-based methods roughly. Some well-known and widely-used algorithms like the KD-tree \cite{Silpaanan2008Optimised}, $R^*$ tree \cite{beckmann1990r}, VA-file \cite{weber1998quantitative}, Locality Sensitive Hashing (LSH) \cite{Gionis1999Similarity}, and Product Quantization (PQ) \cite{jegou2011product} all belong to the above categories. Some works focus on improving the algorithms (\eg , \cite{HuangFZFN15, AroraSK018, gao2014dsh, liu2014sk, ge2013optimized}), while others focus on optimizing the existing methods according to different platforms and scenarios (\eg , \cite{chen2005robust, de2002efficient, teodoro2014approximate, zheng2016lazylsh}). 

\begin{figure}[t]
\begin{center}
\includegraphics[width=220pt]{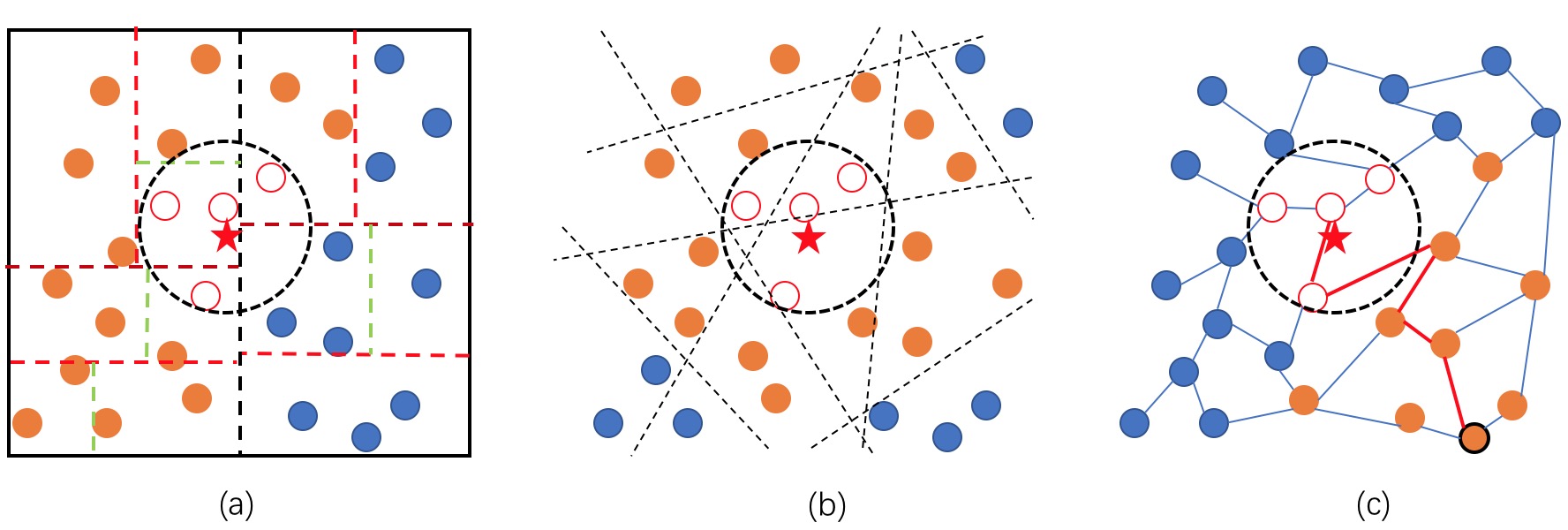}
\end{center}
   \caption{(a) is the tree index, (b) is the hashing index, and (c) is the graph index. The red star is the query (not included in the base data). The four red rings are its nearest neighbors. The tree and hashing index partition the space into several cells. Let each cell contain no more than three points. The out-degree of each node in the graph index is also no more than three. To retrieve the nearest neighbors of the query, we need to backtrack and check many leaf nodes for the tree index. We need to check nearby buckets with hamming radius 2 for the hashing index. As for the graph index, Algorithm 1 forms a search path as the red lines show. The orange circles are checked points during their search. The graph-based algorithm needs the least times of distance calculation.}
\label{tree-hash-graph}
\end{figure}

In the experimental study of some recent works \cite{Jin2014Fast, malkov2014approximate, MalkovYHNSW16, Ben2016Fanng, AroraSK018, CongEfanna2016}, graph-based methods have outperformed some well-known non-graph-based methods (\eg , KD-tree, LSH, PQ) significantly. This may be because the non-graph-based methods all try to solve the ANNS problem by partitioning the space and indexing the resulting subspaces for fast retrieval. Unfortunately, it is not easy to index the subspaces so that neighbor areas can be scanned efficiently to locate the nearest neighbors of a given query. See Figure \ref{tree-hash-graph} as an example (the figure does not include the PQ algorithm because it can be regarded as a hashing method from some perspective).  The non-graph-based methods need to check many nearby cells to achieve high accuracy. A large number of distant points are checked and this problem becomes much worse as the dimension increases (known as the curse of the dimensionality). Graph-based methods may start from a distant position to the query, but they usually approach the query quickly because all these methods are based on proximity graphs which typically express the neighbor relationship better. 

In summary, non-graph-based methods tend to check much more points than the graph-based methods to achieve the same accuracy. This will be shown in our later experiments. 

\subsection{Graph-Based ANNS Methods}
Given a finite point set $S$ in $E^d$, a graph is a structure composed of a set of nodes (representing the points) and edges which link some pairs of the nodes. A node $p$ is called a neighbor of $q$ if and only if there is an edge between $p$ and $q$. Graph-based ANNS solves the ANNS problem defined above via a graph index. Algorithm \ref{search_alg} is commonly used in most graph-based methods. In past decades, many graphs are designed for efficient ANNS. Here we will introduce several graph structures with appealing theoretical properties.

\textbf{Delaunay Graphs} (or Delaunay Triangulations) are defined as the dual graph of the Voronoi diagram \cite{aurenhammer1991voronoi}. It is shown to be a monotonic search network \cite{kurup1992database}, but the time complexity of high-dimensional ANNS on a Delaunay Graph is high. According to Harwood \etal \cite{Ben2016Fanng}, Delaunay Graphs quickly become almost fully connected at high dimensionality. Thus the efficiency of the search reduces dramatically. GNNS \cite{Hajebi2011Fast} is based on the (approximate) $k$NN graph, which is an approximation of Delaunay Graphs. IEH \cite{Jin2014Fast} and Efanna \cite{CongEfanna2016} are also based on the (approximate) kNN graph. They use hashing and Randomized KD-trees to provide better starting positions for Algorithm \ref{search_alg} on the $k$NN graph. Although they improve the performance, they suffer from large and complex indices. 

Wen \etal \cite{li2016approximate} propose a graph structure called \textit{DPG}, which is built upon an approximate $k$NN graph. They propose an edge selection strategy to cut off half of the edges from the prebuilt $k$NN graph and maximize the average angle among the remaining edges. Finally, they will make compensation on the graph to produce an undirected one. Their intuition is to make the angles among edges to distribute evenly around each node, but it lacks theoretical support. According to our empirical study, the DPG suffers from a large index and inferior search performance. 

\textbf{Relative Neighborhood Graphs} (RNG) \cite{toussaint1980relative} are not designed for the ANNS problem in the first place. However, RNG has shown great potential in ANNS. The RNG adopts an interesting edge selection strategy to eliminate the longest edge in all the possible triangles on $S$. With this strategy, the RNG reduces its average out-degree to a constant $C_d+o(1)$, which is only related to the dimension $d$ and usually very small \cite{jaromczyk1992relative}. However, according to Dearholt \etal's study \cite{dearholt1988monotonic}, the RNG does not have sufficient edges to be a monotonic search network due to the strict edge selection strategy. Therefore there is no theoretical guarantee on the length of the path. Dearholt \etal ~proposed a method to add edges to the RNG and turn it into a Monotonic Search Network (MSNET) with the minimal amount of edges, named as the minimal MSNET \cite{dearholt1988monotonic}. The algorithm is based on a prebuilt RNG, the indexing complexity of which is $O(n^{2-\frac{2}{1+d} +\epsilon})$, under the general position assumption \cite{jaromczyk1992relative}. The preprocessing of building the minimal MSNET is of \textit{O}($n^2\log n + n^3$) complexity. The total indexing complexity of the minimal MSNET is huge for high-dimensional and massive databases. Recent practical graph-based methods like FANNG \cite{Ben2016Fanng} and HNSW \cite{MalkovYHNSW16} adopt the RNG's edge selection strategy to reduce the out-degree of their graphs and improve the search performance. However, they did not provide a theoretical analysis.

\textbf{Navigable Small-World Networks} \cite{kleinberg2000navigation, boguna2009navigability} are suitable for the ANNS problem by their nature. The degree of the nodes and the neighbors of each node are all assigned according to a specific probability distribution. The length of the search path on this graph grows polylogarithmically with the network size, \textit{O}($A[logN]^\nu$), where $A$ and $\nu$ are some constants. This is an empirical estimation, which hasn't been proved. Thus the total empirical search complexity is \textit{O}($AD[logN]^\nu$), $D$ is the average degree of the graph. The degree of the graph needs to be carefully chosen, which has a great influence on the search efficiency. Like the other traditional graphs, the time complexity of building such a graph is about \textit{O}($n^2$) in a naive way, which is impractical for massive problems. Yury \etal \cite{malkov2014approximate} proposed NSW graphs to approximate the Navigable Small-World Networks and the Delaunay Graphs simultaneously. But soon they found that the degree of the graph was too high and there also existed connectivity problems in their method. They later proposed HNSW \cite{MalkovYHNSW16} to address this problem. Specifically, they stacked multiple NSWs into a hierarchical structure to solve the connectivity problem. The nodes in the upper layers are sampled through a probability distribution, and the size of the NSWs shrinks from bottom to top layer by layer. Their intuition is that the upper layers enable long-range short-cuts for fast locating of the destination neighborhood. Then they use the RNG's edge selection strategy to reduce the degree of their graphs. HNSW is the most efficient ANNS algorithm so far, according to some open source benchmarks on GitHub\footnote{https://github.com/erikbern/ann-benchmarks}. 

\textbf{Randomized Neighborhood Graphs} \cite{arya1993approximate} are designed for ANNS problem in high-dimensional space. It is constructed in a randomized way. They first partition the space around each node with a set of convex cones, then they select \textit{O}(log $n$) closest nodes in each cone as its neighbors. They prove that the search time complexity on this graph is \textit{O}((log $n$)$^3$), which is very attractive. However, its indexing complexity is too high. To reduce the indexing complexity, they propose a variant, called RNG$^*$. The RNG$^*$ also adopts the edge selection strategy of RNG and uses additional structures (KD-trees) to improve the search performance. However, the time complexity of its indexing is still as high as \textit{O}($n^2$) \cite{arya1993approximate}.

\section{Algorithms And Analysis}
\subsection{Motivation}
\label{Motivation}
The heuristic search algorithm, Algorithm \ref{search_alg}, has been widely used on various graph indices in previous decades. The algorithm walks over the graph and tries to reach the query greedily. Thus, two most crucial factors influencing the search efficiency are the number of greedy hops between the starting node and the destination and the computational cost to choose the next node at each step. In other words, the search time complexity on a graph can be written as \textit{O}($ol$), where $o$ is the average out-degree of the graph and $l$ is the length of the search path. 

In recent graph-based algorithms \cite{Hajebi2011Fast, Ben2016Fanng, malkov2014approximate, li2016approximate, Jin2014Fast, CongEfanna2016, MalkovYHNSW16}, the out-degree of the graph is a tunable parameter. In our experimental study, given a dataset and an expected search accuracy, we find there exist optimal degrees that result in optimal search performance. A possible explanation is that, given an expected accuracy, $ol$ is a convex function of $o$, and the minima of $ol$ determines the search performance of a given graph. In the high accuracy range, the optimal degrees of some algorithms (\eg, GNNS\cite{Hajebi2011Fast}, NSW \cite{malkov2014approximate}, DPG \cite{li2016approximate}) are very large, which leads to very large graph size. The minima of their $ol$ are also very large, leading to inferior performance. Other algorithms \cite{Jin2014Fast, CongEfanna2016, MalkovYHNSW16} use extra index structures to improve their start position in Algorithm \ref{search_alg} in order to minimize $l$ directly. But this also leads to large index size.

From our perspective, we can improve the ANNS performance of graph-based methods by minimizing $o$ and $l$ simultaneously. Moreover, we need to make the index as small as possible to handle large-scale data. What is always ignored is that one should first ensure the existence of a path from the starting node to the query. Otherwise, the targets will never be reached. In summary, we aim to design a graph index with high ANNS performance from the following four aspects. \textbf{(1) ensuring the connectivity of the graph, (2) lowering the average out-degree of the graph, (3) shortening the search path, and (4) reducing index size.} Point (1) is easy to achieve. If the starting node varies with the query, one should ensure that the graph is strongly connected. If the starting node is fixed, one should ensure all other nodes are reachable by a DFS from the starting node. As for point (2)-(4), we address these points simultaneously by designing a better sparse graph for ANNS problem. Below we will propose a new graph structure called Monotonic Relative Neighborhood Graph (MRNG) and a theoretical analysis of its important properties, which leads to better ANNS performance. 

\subsection{Graph Monotonicity And Path Length}
The speed of ANNS on graphs is mainly determined by two factors, the length of the search path and the average out-degree of the graph. Our goal is to find a graph with both low out-degrees and short search paths. We will begin our discussion with how to design a graph with very short search paths. Before we introduce our proposal, we will first provide a detailed analysis of a category of graphs called \textit{Monotonic Search Networks (MSNET)}, which are first discussed in \cite{dearholt1988monotonic} and have shown great potential in ANNS. Here we will present the definition of the MSNETs. 
\subsubsection{Definition And Notation}
Given a point set $S$ in $E^d$ space, $p, q$ are any two points in $S$. Let $B(p, r)$ denote an \textbf{open sphere} such that $B(p, r) = \{x|\delta(x,p) < r\}$. Let $\overset{\longrightarrow}{pq}$ denote a directed edge from $p$ to $q$. 

First we give a definition of the monotonic path in a graph as follows:
\begin{definition}[\textbf{Monotonic Path}]
Given a finite point set $S$ of $n$ points in space $E^d$, $p, q$ are any two points in $S$ and $G$ denotes a graph defined on $S$. Let $v_1, v_2, ... , v_k, (v_1=p, v_k=q)$ denote a path from $p$ to $q$ in $G$, \ie, $\forall{i = 1,...,k-1}$, edge $\overset{\longrightarrow}{v_iv_{i+1}} \in G$. This path is a monotonic path if and only if $\forall{i = 1,...,k-1}, \delta(v_i, q) > \delta(v_{i+1},q)$.
\end{definition}
Then the monotonic search network can be defined as follows:
\begin{definition}[\textbf{Monotonic Search Network}]
Given a finite point set $S$ of $n$ points in space $E^d$, a graph defined on $S$ is a monotonic search network if and only if there exists at least one monotonic path from $p$ to $q$ for any two nodes $p,q \in S$.
\end{definition}


\subsubsection{Analysis On Monotonic Search Networks}

\begin{figure}[t]
\begin{center}
\includegraphics[width=160pt]{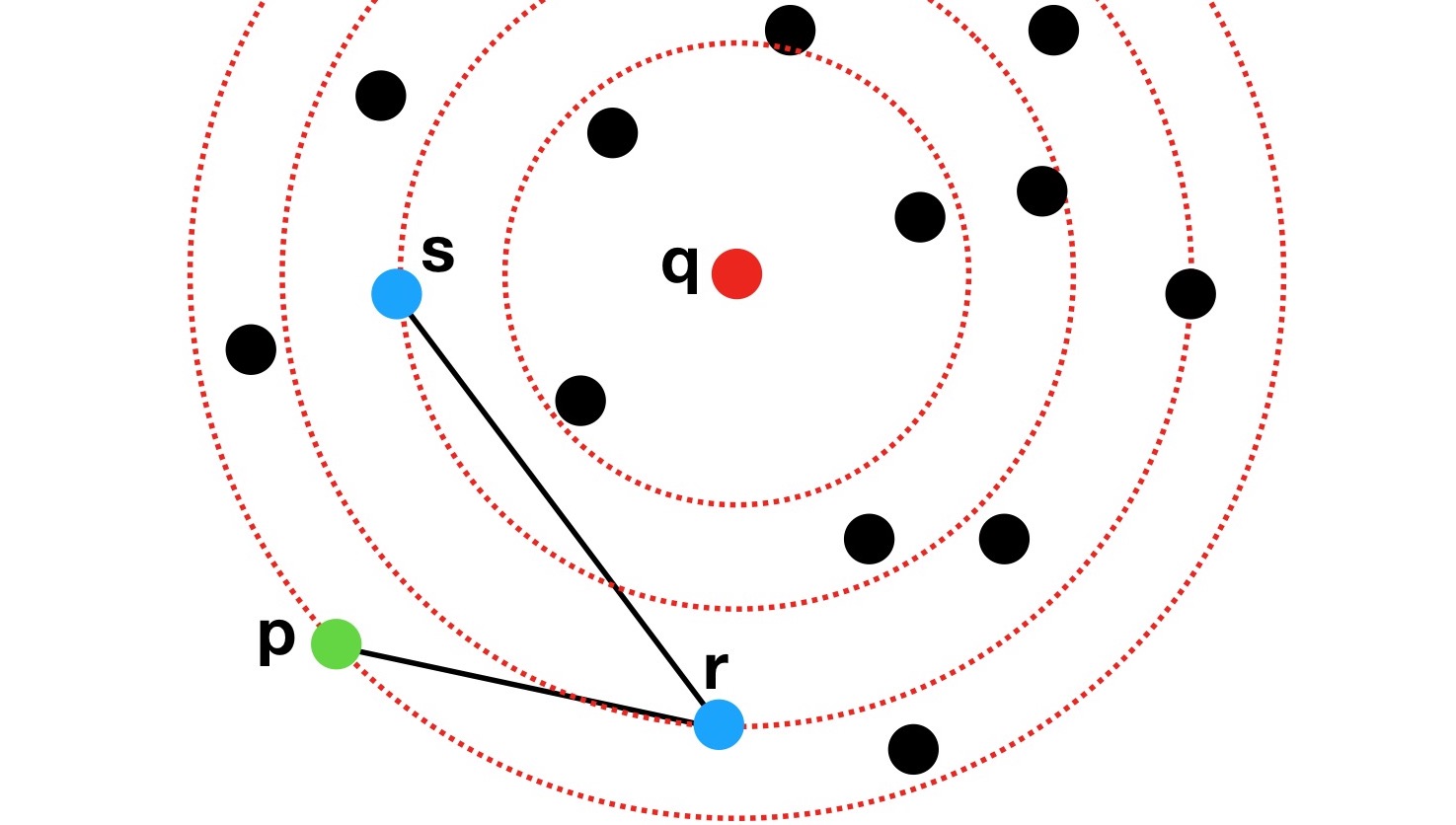}
\end{center}
   \caption{An illustration of the search in an MSNET. The query point is $q$ and the search starts with point $p$. At each step, Algorithm \ref{search_alg} will select a node that is the closest to $q$ among the neighbors of the current nodes. Suppose $p,r,s$ is on a monotonic path selected by Algorithm \ref{search_alg}. The search region shrinks from sphere $B(q,\delta(p,q))$ to $B(q,\delta(r,q))$, then to $B(q,\delta(s,q))$. The number of nodes in each sphere (may be checked) decreases by some ratio at each step until only $q$ is left in the final sphere.}
\label{monotonic}
\end{figure}

The Monotonic Search Networks (MSNET) \cite{dearholt1988monotonic} are a category of graphs which can guarantee a monotonic path between any two nodes in the graph. MSNETs are strongly connected graphs by nature, which ensures the connectivity. When traveling on a monotonic path, we always make progress to the destination at each step. In an MSNET, Dearholt \etal ~hypothesized that one may be able to use Algorithm \ref{search_alg} (commonly used in graph-based search) to detect the monotonic path to the destination node, \ie , no backtracking is needed \cite{dearholt1988monotonic}, which is a very attractive property. Backtracking means, when the algorithm cannot find a closer neighbor to the query (\ie, a local optimal), we need to go back to the visited nodes and find an alternative direction to move on. The monotonicity of the MSNETs makes the search behavior of Algorithm \ref{search_alg} on the graph almost definite and analyzable. However, Dearholt \etal \cite{dearholt1988monotonic} failed to provide a proof of this property. In this section, we will give a concrete proof of this property.


\begin{theorem}
\label{MSNET_find}
Given a finite point set $S$ of $n$ points, randomly distributed in space $E^d$ and a monotonic search network $G$ constructed on $S$, a monotonic path between any two nodes $p,q$ in $G$ can be found by Algorithm \ref{search_alg} without backtracking. 
\end{theorem}
\begin{proof}
Please see the detailed proof in the appendix. 
\end{proof}

From Theorem \ref{MSNET_find}, we know that we can reach the query $q\in S$ on a given MSNET with Algorithm \ref{search_alg} without backtracking, Therefore, the expectation of the iteration number is the same as the length expectation of a monotonic path in the MSNET. Before we discuss the length expectation of a monotonic path in a given MSNET, we first define the MSNETs from a different perspective, which will help with the analysis.

\begin{lemma}
\label{NewDefMSNET}
Given a graph $G$ on a set $S$ of $n$ points in $E^d$, $G$ is an MSNET if and only if for any two nodes $p,q$, there is at least one edge $\overset{\longrightarrow}{pr}$ such that $r \in B(q, \delta(p,q))$. 
\end{lemma}
\begin{proof}
Please see the detailed proof in the appendix.
\end{proof}

From Lemma \ref{NewDefMSNET} we can calculate the length expectation of the monotonic path in the MSNETs as follows. 
\begin{theorem}
\label{MSNETPathLength}
Let $S$ be a finite point set of $n$ points uniformly distributed in a finite subspace in $E^d$. Suppose the volume of the minimal convex hull containing $S$ is $V_S$. The maximal distance between any two points in $S$ is $R$. We impose a constraint on $V_S$ such that when $d$ is fixed, $\exists{\kappa},  \kappa V_S\ge V_B(R)$, where $\kappa$ is a constant independent of $n$, and $V_B(R)$ is the volume of the sphere with radius $R$. We define $\triangle{r}$ as $\triangle{r} = min\{|\delta(a,b) - \delta(a,c)|, |\delta(a,b) - \delta(b,c)|, |\delta(a,c) - \delta(b,c)| \}$, for all possible non-isosceles triangles $abc$ on $S$. $\triangle{r}$ is a decreasing function of $n$.

For a given MSNET defined on such $S$, the length expectation of a monotonic path from $p$ to $q$, for any $p,q\in S$, is $O(n^{1/d}log(n^{1/d}) / \triangle{r})$.
\end{theorem}
\begin{proof}
Please see the detailed proof in the appendix.
\end{proof}

Theorem \ref{MSNETPathLength} is a general property for all kinds of MSNETs. The function $\triangle{r}$ has no definite expression about $n$ because it involves randomness. We have observed that, in practice, $\triangle{r}$ decreases very slowly as $n$ increases. In experiments, we estimate the function of $\triangle{r}$ on different public datasets, based on the proposed graph in this paper. We find that $\triangle{r}$ is mainly influenced by the data distribution and data density. Results are shown in the experiment section. 

Because $O(n^{\frac{1}{d}})$ increases very slowly when $n$ increases in high dimensional space, the length expectation of the monotonic paths in an MSNET, $O(n^{1/d}log(n^{1/d}) / \triangle{r})$, will have a growth rate very close to $O(\log n)$. This is also verified in our experiments. Likewise, we can see that the growth rate of the length expectation of the monotonic paths is lower when $d$ is higher.

In Theorem \ref{MSNETPathLength}, the assumption on the volume of the minimal convex hull containing the data points is actually a constraint on the data distribution. We try to avoid the special shape of the data distribution (\eg, all points form a straight line), which may influence the conclusion. For example, if the data points are all distributed uniformly on a straight line, the length expectation of the monotonic paths on such a dataset will grow almost linearly with $n$. 

In addition, though we assume a uniform distribution of the data points, the property still holds to some extent on other various distributions in practice. Except for some extremely special shape of the data distribution, we can usually expect that, as the search sphere shrinks at each step of the search path, the amount of the nodes remaining in the sphere decreases by some ratio. The ratio is mainly determined by the data distribution, as shown in Figure \ref{monotonic}.

In addition to the length expectation of the search path, another important factor that influences the search complexity is the average out-degree of the graph. The degree of some MSNETs, like the Delaunay Graphs, grows when $n$ increases \cite{aurenhammer1991voronoi}. There is no unified geometrical description of the MSNETs, therefore there is no unified conclusion about how the out-degree of the MSNETs scales. 

Dearholt \etal \cite{dearholt1988monotonic} claim that they have found a way to construct an MSNET with a minimal out-degree. However, there are two problems with their method. Firstly, they did not provide analysis on how the degree of their MSNET scales with $n$. This is mainly because the MSNET they proposed is built by adding edges to an RNG and lacks a geometrical description. Secondly, the proposed MSNET construction method has a very high time complexity (at least $O(n^{2-\frac{2}{1+d} +\epsilon} + n^2\log n + n^3)$) and is not practical in real large-scale scenarios. Below we will propose a new type of MSNET with lower indexing complexity and constant out-degree expectation (independent of $n$). Simply put, the search complexity on this graph scales with $n$ in the same speed as the length expectation of the monotonic paths.

\subsection{Monotonic Relative Neighborhood Graph}
\label{MRNG}
\begin{figure}[t]
\begin{center}
\includegraphics[width=200pt]{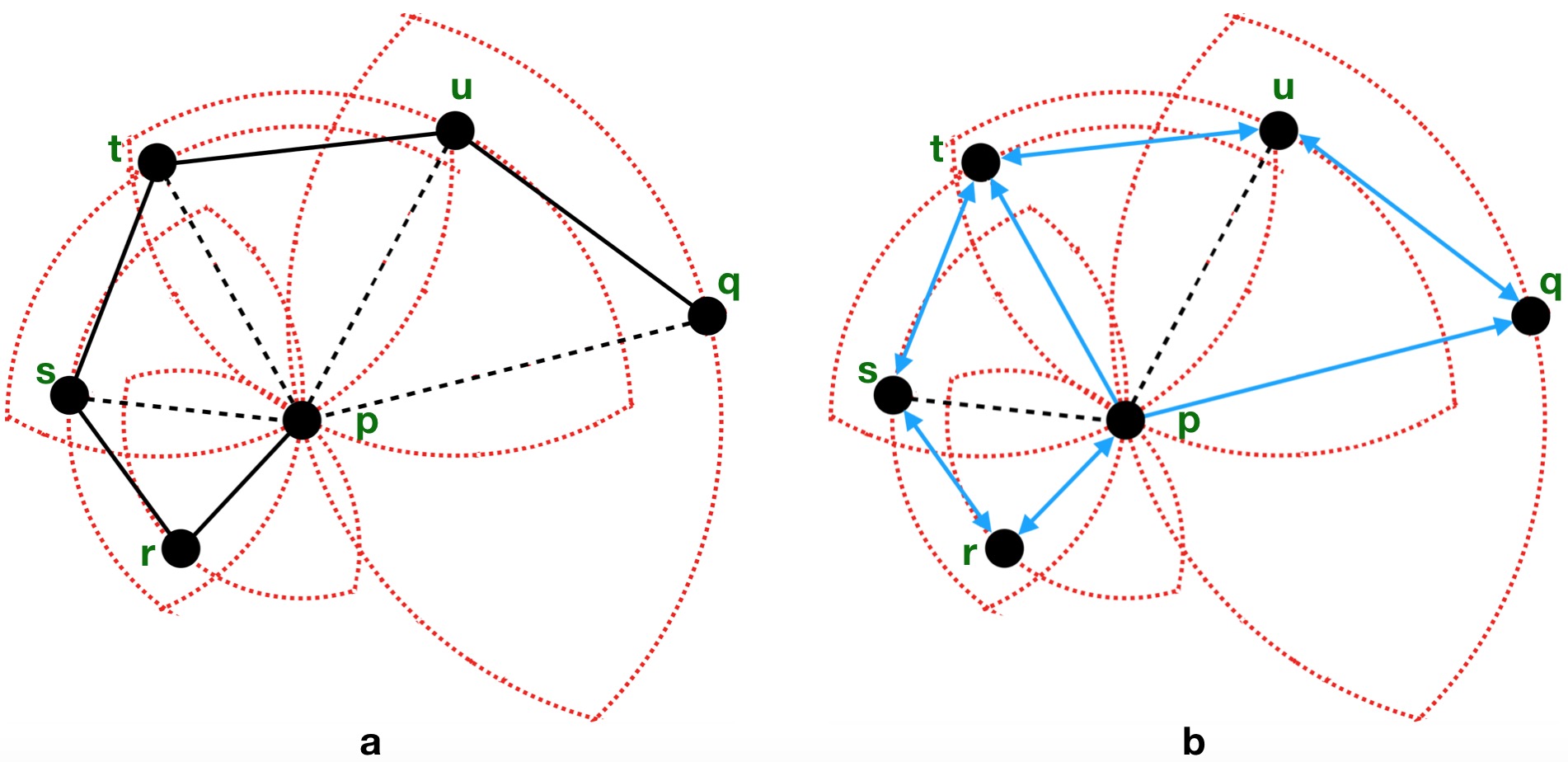}
\end{center}
   \caption{A comparison between the edge selection strategy of the RNG (a) and the MRNG (b). An RNG is an undirected graph, while an MRNG is a directed one. In (a), $p$ and $r$ are linked because there is no node in $lune_{pr}$. Because $r \in lune_{ps}$, $s \in lune_{pt}$, $t \in lune_{pu}$, and $u \in lune_{pq}$, there are no edges between $p$ and $s,t,u,q$. In (b), $p$ and $r$ are linked because there is no node in $lune_{pr}$. $p$ and $s$ are not linked because $r \in lune_{ps}$ and $pr, sr \in$ \textit{MRNG}. Directed edge $\overset{\longrightarrow}{pt} \in$ \textit{MRNG} because $\overset{\longrightarrow}{ps} \notin$ \textit{MRNG}. However, $\overset{\longrightarrow}{tp} \notin$ \textit{MRNG} because $\overset{\longrightarrow}{ts} \in$ \textit{MRNG}. We can see that the MRNG is defined in a recursive way, and the edge selection strategy of the RNG is more strict than MRNG's. In the RNG(a), there is a monotonic path from $q$ to $p$, but no monotonic path from $p$ to $q$. In the MRNG(b), there is at least one monotonic path from any node to another node.}
\label{rulecompare}
\end{figure}

In this section, we describe a new graph structure for ANNS called as MRNG, which belongs to the MSNET family. To make the graph sparse, HNSW and FANNG turn to RNG\cite{toussaint1980relative}, but it was proved that the RNG does not have sufficient edges to be an MSNET\cite{dearholt1988monotonic}. Therefore there is no theoretical guarantee of the search path length in an RNG, and the search on an RNG may suffer from long detours.

Consider the following example. Let $lune_{pq}$ denote a region such that $lune_{pq} = B(p,\delta(p,q)) \cap B(q,\delta(p,q)) $\cite{jaromczyk1992relative}. Given a finite point set $S$ of $n$ points in space $E^d$, for any two nodes $p,q \in S$, edge $pq \in$ RNG if and only if $lune_{pq} \cap S = \emptyset$. In Figure \ref{rulecompare}, (a) is an illustration of a non-monotonic path in an RNG. Node $s$ is in $lune_{pr}$, so $p, s$ are not connected. Similarly, $t, u, q$ are not connected to $p$. When the search goes from $p$ to $q$, the path is non-monotonic (\eg, $rq > pq$). 

We find that this problem is mainly due to RNG's edge selection strategy. Dearholt \etal ~tried to add edges to the RNG \cite{dearholt1988monotonic} to produce an MSNET with the fewest edges, but this method is very time-consuming. Instead, inspired by the RNG, we propose a new edge selection strategy to construct monotonic graphs. The resulting graph may not be the minimal MSNET but it is very sparse. Based on the new strategy, we propose a novel graph structure called Monotonic Relative Neighborhood Graph (MRNG). Formally, an MRNG can be defined as follows:



\begin{figure}[t]
\begin{center}
\includegraphics[width=150pt]{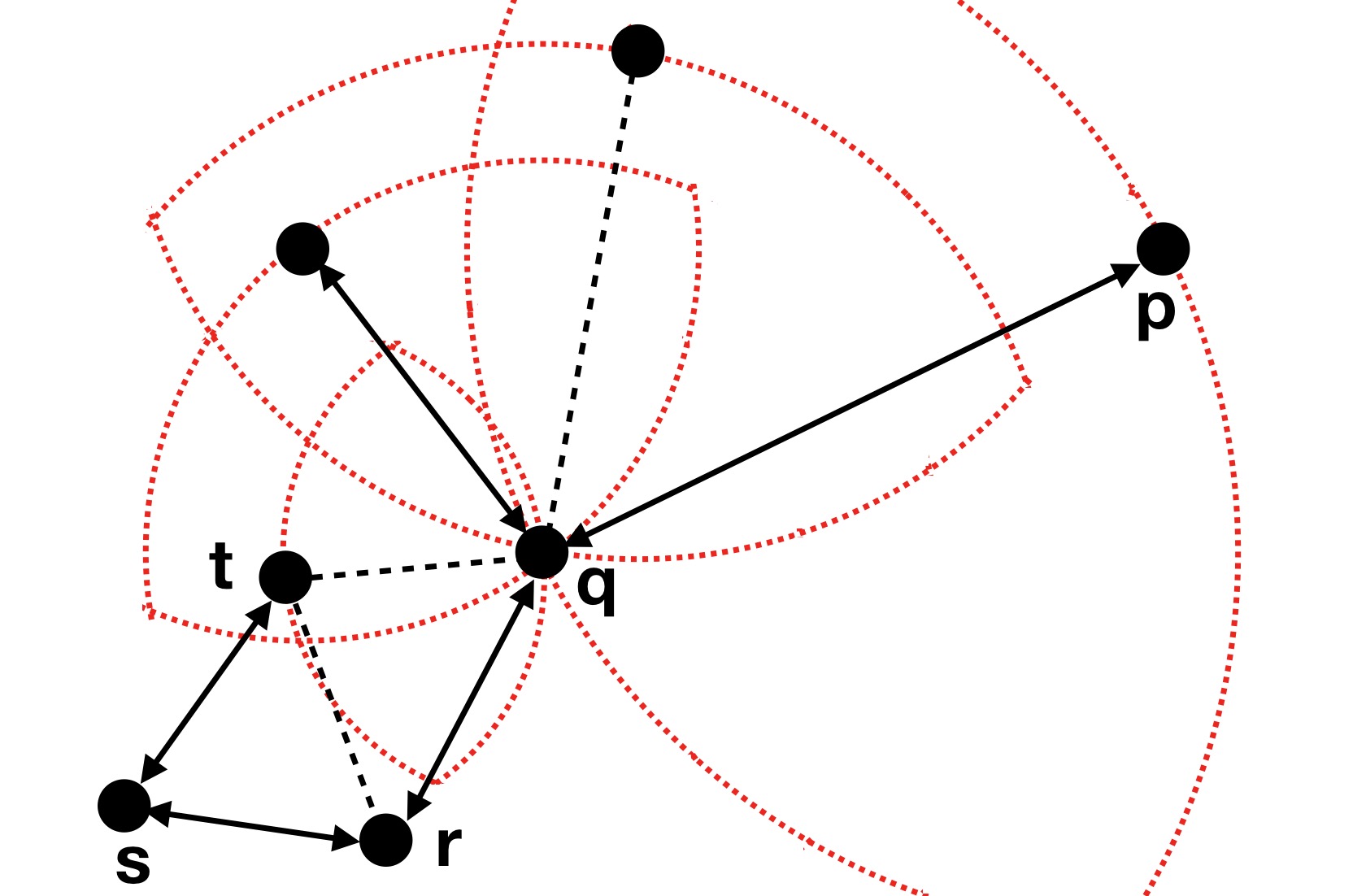}
\end{center}
   \caption{An illustration of the necessity that \textit{NNG} $\subset$ \textit{MRNG}. If not, the graph cannot be an MSNET. Path $p,q,r,s,t$ is an example of non-monotonic path from $p$ to $t$. In this graph, $t$ is the nearest neighbor of $q$ but not linked to $q$. We apply the MRNG's edge selection strategy on this graph. According to the definition of the strategy, $t$ and $r$ can never be linked. When the search goes from $p$ to $t$, it must detour with at least one more step through $s$. This problem will be worse in practice.}
\label{nndetour}
\end{figure}

\begin{definition}[MRNG]
\label{MRNGDef}
Given a finite point set $S$ of $n$ points in space $E^d$, an MRNG is a directed graph with the set of edges satisfying the following property: for any edge $\overset{\longrightarrow}{pq}$, $\overset{\longrightarrow}{pq} \in $MRNG if and only if ~$lune_{pq} \cap S = \emptyset$ or $\forall r \in (lune_{pq}\cap S), \overset{\longrightarrow}{pr} \notin$ MRNG.
\end{definition}

We avoid ambiguity in the following way when isosceles triangles appear. If $\delta(p,q) = \delta(p,r)$ and $qr$ is the shortest edge in triangle $pqr$, we select the edge according to a predefined index, \ie , we select $\overset{\longrightarrow}{pq}$ if \textit{index(q)} $<$ \textit{index(r)}. We can see that the MRNG is defined in a recursive way. In other words, Definition \ref{MRNGDef} implies that for any node $p$, we should select its neighbors from the closest to the farthest. The difference between MRNG's edge selection strategy and RNG's is that, for any edge $pq\in MRNG$, $lune_{pq} \cap S$ is not necessarily $\emptyset$. The difference can be seen in Figure \ref{rulecompare} clearly. Here we show that the MRNG is an MSNET.

\begin{theorem}
\label{MRNGProve}
Given a finite point set $S$ of $n$ points. An MRNG defined on $S$ is an MSNET.
\end{theorem}
\begin{proof}
Please see the detailed proof in the appendix.
\end{proof}

Though different in the structure, MRNG and RNG share some common edges. We will start with defining the Nearest Neighbor Graph (NNG) as follows: 
\begin{definition}[NNG]
Given a finite point set $S$ of $n$ points in space $E^d$, an NNG is the set of edges such that, for any edge $\overset{\longrightarrow}{pq}$, $\overset{\longrightarrow}{pq} \in $ NNG if and only if $q$ is the closest neighbor of $p$ in $S$.
\end{definition}

Similarly, we can remove the ambiguity in the NNG by assigning a unique index for each node and linking the node to its nearest neighbor with the smallest index. Obviously, we have \textit{MRNG} $\cap$ \textit{RNG} $\supset$ \textit{NNG} (if a node $q$ is the nearest neighbor of $p$, we have $lune_{pq} \cap S = \emptyset$). This is necessary for MRNG's monotonicity. Figure \ref{nndetour} shows an example of the non-monotonic path if we apply MRNG's edge selection strategy on some graph $G$ but do not guarantee \textit{NNG} $\subset$ $G$. The edges in Figure \ref{nndetour} satisfy the selection strategy of the MRNG except that $q$ is forced not to be linked to its nearest neighbor $t$. Because $t$ is the nearest neighbor of $q$, we have $\delta(q,r) > \delta(q,t)$. Because $qt$ is the shortest edge in triangle $qtr$ and $q,r$ is linked, then $rt$ must be the longest edge in triangle $qtr$ according to the edge selection strategy of MRNG. Thus, $r,t$ can not be linked and we can only reach $t$ through other nodes (like $s$). Similarly, only when $rt$ is the longest edge in triangle $rst$, edge $rs$ and $st$ can coexist in this graph. Therefore when we go from $p$ to $t$, we need a detour at least via nodes $r,s$. Because $\delta(q,r) > \delta(q,t)$, it is a non-monotonic path from $p$ to $t$. If we don't guarantee \textit{NNG} $\subset$ \textit{MRNG}, detours are unavoidable, which may be worse in practice. It's easy to verify that a similar detour problem will also appear if we perform RNG's edge selection strategy on $G$ but do not guarantee that \textit{NNG} $\subset$ \textit{G}.

\begin{figure}[t]
\begin{center}
\includegraphics[width=140pt]{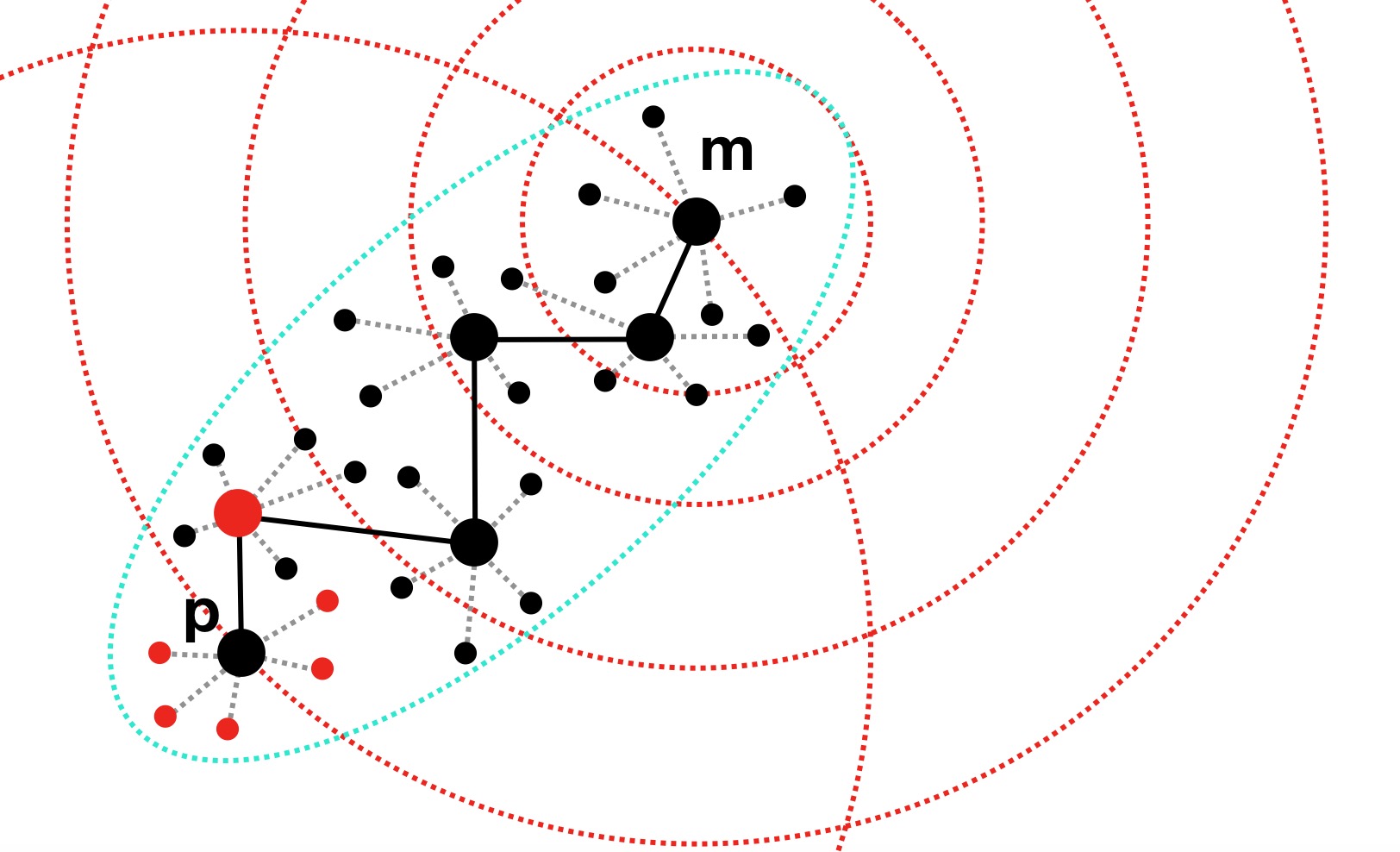}
\end{center}
   \caption{An illustration of the candidates of edge selection in NSG. Node $p$ is the node to be processed, and $m$ is the Navigating Node. The red nodes are the $k$ nearest neighbors of node $p$. The big black nodes and the solid lines form a possible monotonic path from $m$ to $p$, generated by the \textbf{search-and-collect} routine. The small black nodes are the nodes visited by the search-and-collect routine. All the nodes in the figure will be added to the candidate set of $p$.}
\label{candidate}
\end{figure}

Here we will discuss the average out-degree of the MRNG. The MRNG has more edges than the RNG, but it's still very sparse because the angle between any two edges sharing the same node is at least $60^\circ$ (by the definition of MRNG, for any two edges $pq,pr \in $ \textit{MRNG}, $qr$ must be the longest edge in triangle $pqr$ and $qr \notin$ \textit{MRNG}).
\begin{lemma}
\label{MRNGConstantDegree}
Given an MRNG in $E^d$, the max degree of the MRNG is a constant and independent of $n$.
\end{lemma}
\begin{proof}
Please see the detailed proof in the appendix.
\end{proof}

Now according to Lemma \ref{MRNGConstantDegree}, Theorem \ref{MSNET_find}, Theorem \ref{MSNETPathLength}, and Theorem \ref{MRNGProve}, we have that the search complexity in expectation on an MRNG is $O(cn^{\frac{1}{d}}\log n^{\frac{1}{d}}/\triangle{r})$, where $c$ is the average degree of the MRNG and independent of $n$, $\triangle{r}$ is a function of $n$, which decreases very slowly as $n$ increases. 


\subsection{MRNG Construction}
The MRNG can be constructed simply by applying our edge selection strategy on each node. Specifically, for each node $p$, we denote the set of rest nodes in $S$ as $R = S\backslash\{p\}$. We calculate the distance between each node in $R$ and $p$, then rank them in ascending order according to the distance. We denote the selected node set as $L$. We add the closest node in $R$ to $L$ to ensure \textit{NNG} $\subset$ \textit{MRNG}. Next, we fetch a node $q$ from $R$ and a node $r$ from $L$ in order to check whether $pq$ is the longest edge in triangle $pqr$. If $pq$ is not the longest edge in triangle $pqr, \forall r \in L$ , we add $q$ to $L$. We repeat this process until all the nodes in $R$ are checked. This naive construction runs in $O(n^2\log n + n^2c)$ time, where $c$ is the average out-degree of MRNG, which is much smaller than that of the MSNET indexing method proposed in \cite{dearholt1988monotonic}, which is at least $O(n^{2-\frac{2}{1+d} +\epsilon} + n^2\log n + n^3)$ under the general position assumption.

\begin{algorithm}[t]\small
	\caption{NSGbuild($G$, $l$, $m$)}
	\label{NSGbuild_alg}
	\begin{algorithmic}[1]
		\Require $k$NN Graph $G$, candidate pool size $l$ for greedy search, max-out-degree $m$.
		\Ensure NSG with navigating node $\textbf{n}$
		\State calculate the centroid $\textbf{c}$ of the dataset.
		\State $\textbf{r} = $random node.
		\State $\textbf{n} = $Search-on-Graph($G$,$\textbf{r}$,$\textbf{c}$,$l$)\%navigating node
		\ForAll{node $\textbf{v}$ in G}
		\State Search-on-Graph($G$,$\textbf{n}$,$\textbf{v}$,$l$)
		\State $E = $all the nodes checked along the search
		\State add $\textbf{v}$'s nearest neighbors in $G$ to $E$ 
		\State sort $E$ in the ascending order of the distance to $\textbf{v}$.
		\State result set $R = \emptyset$, $\textbf{p}_0 = $ the closest node to $\textbf{v}$ in $E$
		\State $R$.add($\textbf{p}_0$)
		\While{!$E$.empty() \&\& $R$.size() $< m$}
		\State $\textbf{p} = E$.front()
		\State $E$.remove($E$.front())
		\ForAll{node $\textbf{r}$ in $R$}
		\If{edge $\textbf{pv}$ conflicts with edge $\textbf{pr}$}
		\State break
		\EndIf
		\EndFor
		\If{no conflicts occurs}
		\State $R$.add($\textbf{p}$)
		\EndIf
		\EndWhile
		\EndFor
		\While{True}
		\State build a tree with edges in NSG from root \textbf{n} with DFS.
		\If{not all nodes linked to the tree}
		\State add an edge between one of the out-of-tree nodes and \State its closest in-tree neighbor (by algorithm 1).
		\Else
		\State break.
		\EndIf
		\EndWhile
	\end{algorithmic}
\end{algorithm}

\subsection{NSG: A Practical Approximation For MRNG}
\label{NSG}
Though MRNG can guarantee a fast search time, its high indexing time is still not practical for large-scale problems. In this section, we will present a practical approach by approximating our MRNG and starting from the four criteria to design a good graph for ANNS. We name it the \textbf{Navigating Spreading-out Graph} (NSG). We first present the NSG construction algorithm (Algorithm \ref{NSGbuild_alg}) as follows: 

\begin{enumerate}[i]
\setlength{\itemsep}{1pt}
\setlength{\parskip}{0pt}
\setlength{\parsep}{0pt}
\item We build an approximate $k$NN graph with the current state-of-the-art methods(\eg, \cite{johnson2017billion, Dong2011Efficient}).
\item We find the approximate medoid of the dataset. This can be achieved by the following steps. (1) Calculate the centroid of the dataset; (2) Treat the centroid as the query, search on the $k$NN graph with Algorithm \ref{search_alg}, and take the returned nearest neighbor as the approximate medoid. This node is named as the Navigating Node because all the search will start with this fixed node.
\item For each node, we generate a candidate neighbor set and select neighbors for it from the candidate sets. This can be achieved by the following steps. For a given node $p$, (1) we treat it as a query and perform Algorithm \ref{search_alg} starting from the Navigating Node on the prebuilt $k$NN graph. (2) During the search, each visited node $q$ (\ie , the distance between $p$ and $q$ is calculated) will be added to the candidate set (the distance is also recorded). (3) Select at most $m$ neighbors for $p$ from the candidate set with the edge selection strategy of MRNG. 
\item We span a Depth-First-Search tree on the graph produced in previous steps. We treat the Navigating Node as the root. When the DFS terminates, and there are nodes which are not linked to the tree, we link them to their approximate nearest neighbors (from Algorithm \ref{search_alg}) and continue the DFS. 
\end{enumerate}

What follows is the motivation of the NSG construction algorithm. The ultimate goal is to build an approximation of MRNG with low indexing time complexity.

\textbf{(i)} MRNG ensures there exists at least one monotonic path between any two nodes, however, it is not an easy task. Instead, we just pick one node out and try to guarantee the existence of monotonic paths from this node to all the others. We name this node as the Navigating Node. When we perform the search, we always start from the Navigating Node, which makes the search on an NSG almost as efficient as on an MRNG. 

\textbf{(ii)} The edge selection strategy of the MRNG treats all the other nodes as candidate neighbors of the current node, which causes a high time complexity. To speed up this process, we want to generate a small subset of candidates for each node. These candidates contain two parts: (1) As discussed above, the NNG is essential for monotonicity. Because it is very time-consuming to get the exact NNG, we turn to the approximate kNN graph. A high quality approximate kNN graph usually contains a high quality approximate NNG. It is acceptable when only a few nodes are not linked to their nearest neighbors. (2) Because the search on the NSG always starts from the Navigating Node $p_n$, for a given node $p$, we only need to consider those nodes which are on the search path from the $p_n$ to $p$. Therefore we treat $p$ as the query and perform Algorithm \ref{search_alg} on the prebuilt $k$NN graph. The nodes visited by the search and $p$'s nearest neighbor in the approximate NNG are recorded as candidates. The nodes forming the monotonic path from the Navigating Node to $p$ are very likely included in the candidates. When we perform MRNG's edge selection strategy on these candidates, it's very likely that the NSG inherits the monotonic path in the MRNG from the Navigating Node to $p$.

\textbf{(iii)} A possible problem in the above approach is the degree explosion problem for some nodes. Especially, the Navigating Node and nodes in dense areas will act as the ``traffic hubs'' and have high out-degrees. This problem is also discussed in HNSW \cite{MalkovYHNSW16}. They introduced a multi-layer graph structure to solve this problem, but their solution increased the memory usage significantly. Our solution is to limit the out-degrees of all the nodes to a small value $m \ll n$ by abandoning the longer edges. The consequence is the connectivity of the graph is no longer guaranteed due to the edge elimination. 

To address the connectivity problem, we introduce a new method based on the DFS spanning tree as described above. After this process, all the nodes are guaranteed at least one path spreading out from the Navigating Node. Though the proposed method will sacrifice some performance in the worst case, the detours in the NSG will be minimized if we build a high-quality approximate $k$NN graph and choose a proper degree limitation $m$. 

By approximating the MRNG, the NSG can inherit similar \textbf{low} search complexity as the MRNG. Meanwhile, the degree upper-bound makes the graph very \textbf{sparse}, and the tree-spanning operation guarantees the \textbf{connectivity} of the NSG.  The NSG's index contains only a sparse graph and no auxiliary structures. Our method has made progress on all the four aspects compared with previous works. These improvements are also verified in our experiments. The detailed results will be presented in the later sections. 

\subsubsection{Indexing Complexity of NSG}
The total indexing complexity of the NSG contains two parts, the complexity of the $k$NN graph construction and the preprocessing steps of NSG. In the million-scale experiments, we use the $nn$-descent algorithm \cite{Dong2011Efficient} to build the approximate $k$NN graph on CPU. In the DEEP100M experiments, we use Faiss \cite{johnson2017billion} to build it on GPU because the memory consumption of $nn$-descent explodes on large datasets. We focus on the complexity of Algorithm \ref{NSGbuild_alg} in this section.

The preprocess steps of NSG include the search-collect-select operation and the tree spanning. Because the $k$NN graph is an approximation of the Delaunay Graph (an MSNET), the search complexity on it is approximately $O(kn^{\frac{1}{d}}\log n^{\frac{1}{d}}/\triangle r)$. We search for all the nodes, so the total complexity is about $O(kn^{\frac{1+d}{d}}\log n^{\frac{1}{d}}/\triangle r)$. The complexity of the edge selection is $O(nlc)$, where $l$ is the number of the candidates generated by the search and $c$ is the maximal degree we set for the graph. Because $c$ and $l$ are usually very small in practice (\ie , $c\ll n, l\ll n$), this process is very fast. The final process is the tree spanning. This process is very fast because the number of the strongly connected components is usually much smaller than $n$. We only need to add a small number of edges to the graph. We can see that the most time-consuming part is the ``search-collect'' part. Therefore the total complexity of these processes is about $O(kn^{\frac{1+d}{d}}\log n^{\frac{1}{d}}/\triangle r)$, which is verified in our experimental evaluation in later sections. We also find that $\triangle r$ is almost a constant and does not influence the complexity in our experiments. 

In the implementation of this paper, the overall empirical indexing complexity of the NSG is $O(kn^{\frac{1+d}{d}}\log n^{\frac{1}{d}} + f(n))$ ($f(n) = n^{1.16}$ with $nn$-descent and $f(n)=n\log n$ with Faiss), which is much lower than $O(n^2\log n + cn^2)$ of the MRNG.

\subsection{Search On NSG}
We use Algorithm \ref{search_alg} for the search on the NSG, and we always start the search from the Navigating Node. Because the NSG is a carefully designed approximation of the MRNG, the search complexity on the NSG is approximately $O(cn^{\frac{1}{d}}\log n^{\frac{1}{d}}/\triangle r)$ on average, where $c$ is the maximal degree of the NSG, and $d$ is the dimension. In our experiments, $\triangle{r}$ is about $O(n^{-\frac{\epsilon}{d}})$, $0<\epsilon\ll d$. So the empirical average search complexity is $O(cn^{\frac{1+\epsilon}{d}}\log n^{\frac{1}{d}})$. Because $1+\epsilon \ll d$, the complexity is very close to $O(\log n)$, which is verified in our experimental evaluation in later sections. Our code has been released on GitHub\footnote{https://github.com/ZJULearning/nsg}.

\section{Experiments}
In this section, we provide a detailed analysis of extensive experiments on public and synthetic datasets to demonstrate the effectiveness of our approach.

\subsection{Million-Scale ANNS}
\subsubsection{Datasets}
Because not all the recent state-of-the-art algorithms can scale to billion-point datasets, this experiment is conducted on four million-scale datasets. SIFT1M and GIST1M are in the BIGANN datasets\footnote{http://corpus-texmex.irisa.fr/}, which are widely used in related literature \cite{jegou2011product, Ben2016Fanng}. RAND4M and GAUSS5M are two synthetic datasets. RAND4M and GAUSS5M are generated from the uniform distribution $U(0, 1)$ and Gaussian distribution $N(0, 3)$ respectively. Considering that the data may lie on a low dimensional manifold, we measure the local intrinsic dimension (LID) \cite{costa2005estimating} to reflect the datasets' degree of difficulty better. See Table \ref{dataset_info_tb} for more details.   

To prevent the indices from overfitting the query data, we repartition the datasets by randomly sampling one percent of the points out of each training set as a \textbf{validation set}. Since it's essential to be fast in the high-precision region (over 90\%) in real scenarios, we focus on the performance of all algorithms in the high-precision area. We tune their indices on the validation set to get the best performance in the high-precision region. 

\subsubsection{Compared Algorithms}

The algorithms we choose for comparison cover various types such as tree-based, hashing-based, quantization-based and graph-based approaches. The codes of most algorithms are available on GitHub and well optimized. For those who do not release their codes, we implement their algorithms according to their papers. They are implemented in C++, compiled by g++4.9 with ``O3'' option. The experiments of SIFT1M and GIST1M are carried out on a machine with i7-4790K CPU and 32GB memory. The experiments on RAND4M and GAUSS5M are carried out on a machine with Xeon E5-2630 CPU and 96GB memory. The indexing of NSG contains two steps, the $k$NN graph construction and Algorithm \ref{NSGbuild_alg}. We use the $nn$-descent algorithm \cite{Dong2011Efficient} to build the $k$NN graphs.

\begin{table}[t]\scriptsize
\caption{Information on experimental datasets. We list the dimension (D), local intrinsic dimension (LID) \cite{costa2005estimating}, the number of base vectors, and the number of query vectors.}
\label{dataset_info_tb}
\centering
\begin{tabular}{ccccc}
\hline
\hline
dataset &  D & LID &No. of base &No. of query\\
\hline
SIFT1M &  128 & 12.9 & 1,000,000 & 10,000 \\
GIST1M &  960 & 29.1 & 1,000,000 & 1,000 \\
RAND4M & 128 & 49.5 & 4,000,000 & 10,000 \\
GAUSS5M & 128 & 48.1 & 5,000,000 & 10,000 \\
\hline
\hline
\end{tabular}
\end{table}
Because not all algorithms support \textbf{inner-query parallelizing}, for all the search experiments, we only evaluate the algorithms with a single thread. Given that all the compared algorithms have the parallel versions for their index building algorithms, for time-saving, we construct all the indices with eight threads. 
\begin{enumerate}
\setlength{\itemsep}{1pt}
\setlength{\parskip}{0pt}
\setlength{\parsep}{0pt}
\item \textbf{Serial Scan} We perform serial scan on the base data to get the accurate nearest neighbors for the test points. 
\item \textbf{Tree-Based Methods.} \textbf{Flann\footnote{https://github.com/mariusmuja/flann}} is a well-known ANNS library based on randomized KD-tree, K-means trees, and composite tree algorithm. We use its randomized KD-tree algorithm for comparison. \textbf{Annoy\footnote{https://github.com/spotify/annoy}} is based on a binary search forest.
\item \textbf{Hashing-Based Methods.} FALCONN\footnote{https://github.com/FALCONN-LIB/FALCONN} is a well-known ANNS library based on multi-probe locality sensitive hashing.

\item \textbf{Quantization-Based Methods.} \textbf{Faiss\footnote{https://github.com/facebookresearch/faiss}} is recently released by Facebook. It contains well-implemented codes for state-of-the-art product-quantization-based methods on both CPU and GPU. The CPU version is used here for a fair comparison.

\item \textbf{Graph-Based Methods.} \textbf{KGraph\footnote{https://github.com/aaalgo/kgraph}} is based on a $k$NN Graph. \textbf{Efanna\footnote{https://github.com/ZJULearning/efanna}} is based on a composite index of randomized KD-trees and a $k$NN graph. \textbf{FANNG} is based on a kind of graph structure proposed in \cite{Ben2016Fanng}. They did not release their code. Thus, we implement their algorithm according to their paper. \textbf{HNSW\footnote{https://github.com/searchivarius/nmslib}} is based on a hierarchical graph structure, which was proposed in \cite{MalkovYHNSW16}. \textbf{DPG\footnote{https://github.com/DBWangGroupUNSW/nns\_benchmark}} is based on an undirected graph whose edges are selected from a $k$NN graph. According to an open source benchmark\footnote{https://github.com/erikbern/ann-benchmarks}, \textbf{HNSW is the fastest ANNS algorithm on CPU so far.}

\item \textbf{NSG} is the method proposed in this paper. It contains only one graph with a navigating node where the search always starts. 
\item \textbf{NSG-Naive} is a designed baseline to demonstrate the necessity of NSG's search-collect-select operation and the guarantee of the graph connectivity. We directly perform the edge selection strategy of MRNG on the edges of the approximate $k$NN graph to get NSG-Naive. There is no navigating node, thus, we use Algorithm \ref{search_alg} with random initialization on NSG-Naive. 
\end{enumerate}

\subsubsection{Results}

\begin{table}[t]\scriptsize
\caption{Information of the graph-based indices involved in all of our experiments. AOD means the Average Out-Degree. MOD means the Maximum Out-Degree. The NN(\%) means the percentage of the nodes which are linked to their nearest neighbor. Because HNSW contains multiple graphs, we only report the AOD, MOD, and NN(\%) of its bottom-layer graph (HNSW$_0$) here.}
\label{best_index_tb}
\centering
\begin{tabular}{|p{1.2cm}<{\centering}|p{1.2cm}<{\centering}|p{1.5cm}<{\centering}|p{0.7cm}<{\centering}|p{0.6cm}<{\centering}|p{0.7cm}<{\centering}|}
\hline
dataset & algorithms &memory(MB) & AOD & MOD & NN(\%) \\
\hline
\multirow{6}{*}{SIFT1M} 
& NSG & \textbf{153} &25.9& 50 & 99.3 \\
\cline{2-6}
&HNSW$_0$ & 451 &32.1& 50 & 66.3\\
\cline{2-6}
& FANNG & 374 &30.2& 98 & 60.4\\
\cline{2-6}
& Efanna & 1403 &300 &300 & 99.4 \\
\cline{2-6}
& KGraph & 1144&300&300 & 99.4\\
\cline{2-6}
& DPG & 632 & 165.1 & 1260 & 99.4 \\

\hline
\multirow{6}{*}{GIST1M} 
& NSG & \textbf{267} &26.3& 70 & 98.1\\
\cline{2-6}
& HNSW$_0$ & 667 &23.9 & 70 &47.5\\
\cline{2-6}
& FANNG & 1526 & 29.2&400 & 39.9\\
\cline{2-6}
& Efanna & 2154 &400 &400 &98.1\\
\cline{2-6}
& KGraph & 1526 &400& 400 &98.1\\
\cline{2-6}
& DPG & 741 & 194.3 & 20899 & 98.1\\
\hline

\multirow{6}{*}{RAND4M} 
& NSG & \textbf{2.7} $\times$ \textbf{10}$^\textbf{3}$ &174.0 & 220 & 96.4\\
\cline{2-6}
& HNSW$_0$ & $6.7 \times 10^3$ & 161.0 & 220 & 76.5\\
\cline{2-6}
& FANNG & $5.0 \times 10 ^3$ &181.2& 327  &66.7\\
\cline{2-6}
& Efanna & $6.3 \times 10^3$ & 400 &400  &96.6\\
\cline{2-6}
& KGraph & $6.1 \times 10^3$ &400 &400 &96.6\\
\cline{2-6}
& DPG & $4.7 \times 10^3$ & 246.4 & 5309 &96.6\\
\hline

\multirow{6}{*}{GAUSS5M} 
& NSG & \textbf{2.6} $\times$ $\textbf{10}^\textbf{3}$ &146.2 & 220 &94.3\\
\cline{2-6}
& HNSW$_0$ & $6.7 \times 10^3 $ &131.9&220 &57.6\\
\cline{2-6}
& FANNG & $5.2 \times 10^3 $ &152.2& 433 &53.4\\
\cline{2-6}
& Efanna & $7.8 \times 10^3 $ &400&400 &94.3\\
\cline{2-6}
& KGraph & $7.6 \times 10^3 $ &400&400 &94.3\\
\cline{2-6}
& DPG & $3.7 \times 10^3$ & 194.0& 15504 &94.3\\
\hline
\end{tabular}
\end{table}

\textbf{A. Non-Graph-Based v.s. Graph-Based}. We record the numbers of distance calculations of Flann (Randomized KD-trees), FALCONN(LSH), Faiss(IVFPQ), and NSG on SIFT1M and GIST1M to reach certain search precision. In our experiments, at the same precision, The other methods checks tens of times more points than NSG (see the figure in the appendix). This is the main reason of the big performance gap between graph-based methods and non-graph based methods.
\\

\textbf{B. Check Motivation}. In this paper, we aim to design a graph index with high ANNS performance from the following four aspects: (1) ensuring the connectivity of the graph, (2) lowering the average out-degree of the graph and (3) shortening the search path, and (4) reducing index size.

\begin{table}[t]\scriptsize
	\caption{The indexing time of all the graph-based methods. The indexing time of NSG is recorded in the form $t_1+t_2$, where $t_1$ is the time to build the $k$NN graph, and $t_2$ is the time of Algorithm \ref{NSGbuild_alg}.}
	\label{index_time}
	\centering
	\begin{tabular}{|p{1.2cm}<{\centering}|p{1.1cm}<{\centering}|p{1.5cm}<{\centering}|p{1.1cm}<{\centering}|p{1.3cm}<{\centering}|}
		\hline
		dataset &  algorithm & time(s) &algorithm & time(s)\\
		\hline
		\multirow{3}{*}{SIFT1M} 
		&NSG & \textbf{140+134} & HNSW & 376\\
		\cline{2-5}
		&FANNG & 1860 & DPG & 1120 \\
		\cline{2-5}
		&KGraph & 824 & Efanna & 355 \\
		\hline
		\multirow{3}{*}{GIST1M} 
		&NSG & \textbf{1982+2078} & HNSW & \textbf{4010}\\
		\cline{2-5}
		&FANNG & 34530 & DPG & 6700 \\
		\cline{2-5}
		&KGraph & 4300 & Efanna & 4335 \\
		\hline
		dataset &  algorithm & time(h) &algorithm & time(h)\\
		\hline
		\multirow{3}{*}{RAND4M} 
		&NSG & \textbf{2.1+2.5} & HNSW & 5.6\\
		\cline{2-5}
		&FANNG & 38.3 & DPG & 6.0 \\
		\cline{2-5}
		&KGraph & 4.9 & Efanna & 5.1 \\
		\hline
		\multirow{3}{*}{GAUSS5M} 
		&NSG &  \textbf{2.3+2.5} & HNSW & 6.7\\
		\cline{2-5}
		&FANNG & 46.1 & DPG & 6.4 \\
		\cline{2-5}
		&KGraph & 5.1 & Efanna & 5.3 \\
		\hline
	\end{tabular}
\end{table}

\begin{figure*}[t]
	\centering
	\subfigure{\includegraphics[width=410pt]{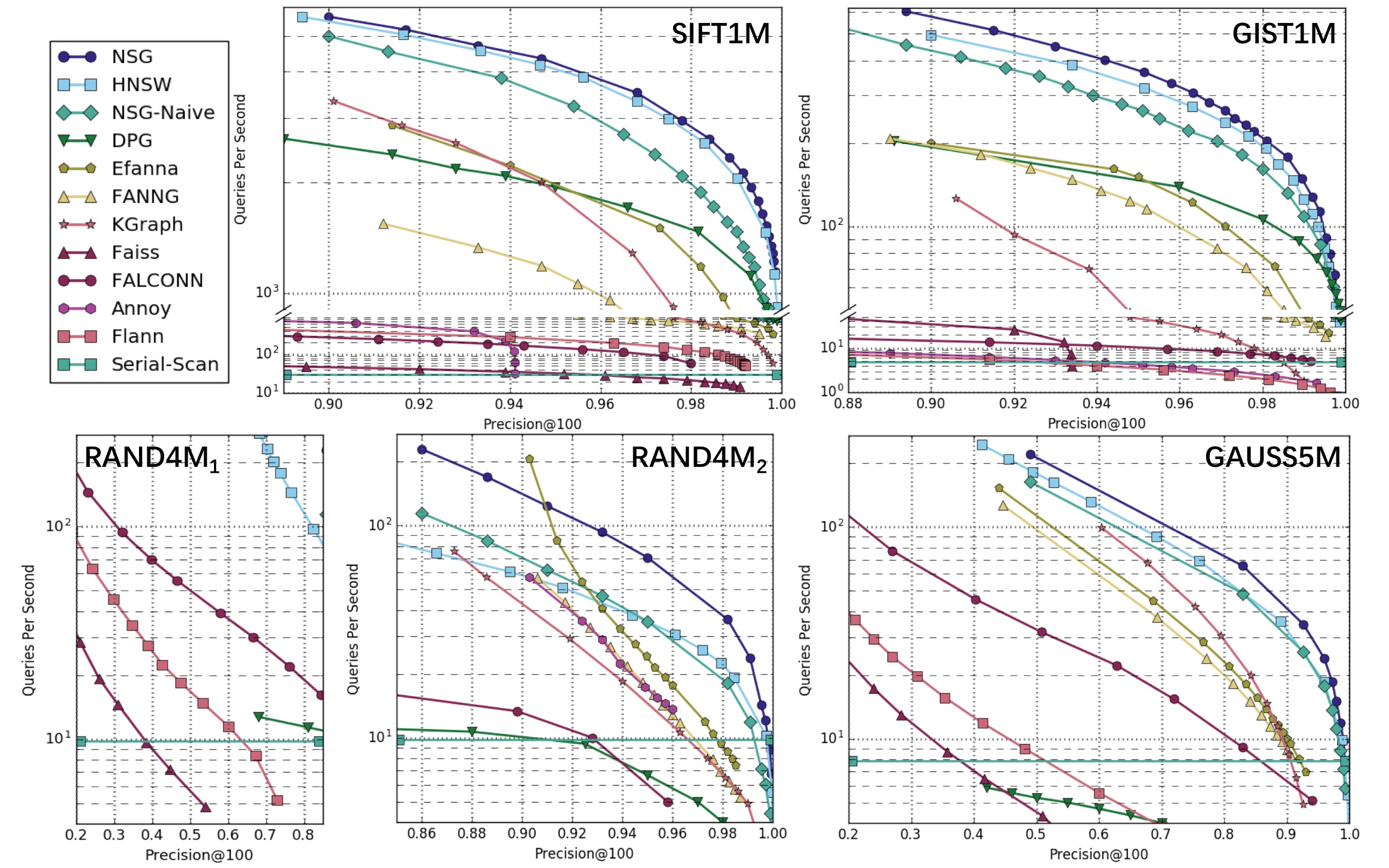}}
	\caption{ANNS performance of graph-based algorithms with their optimal indices in high-precision region on the four datasets (top right is better). Some of the non-graph based methods have much worse performance than the graph-based ones. So we break the y-axis of SIFT1M and GIST1M figures, and we break the x-axis of the RAND4M figure (RAND4M$_1$ and RAND4M$_2$) to provide better view of the curves. The x-axis is not applicable for Serial-Scan because the results are accurate. }
	\label{graph_search}
\end{figure*}
\begin{enumerate}
\setlength{\itemsep}{1pt}
\setlength{\parskip}{0pt}
\setlength{\parsep}{0pt}
\item \textbf{Graph connectivity.} Among the graph-based methods, NSG and HNSW start their search with fixed node. To ensure the connectivity, NSG and HNSW should guarantee the other points are reachable from the fixed starting node. The others should guarantee their graph to be strongly connected because the search could start from any node. In our experiments, we find that, except NSG and HNSW, the rest methods have more than one strongly connected components on some datasets. Only NSG and HNSW guarantee the connectivity over different datasets (see the table in the appendix).

\item \textbf{Lower the out-degree and Shorten the search path.} In \textbf{Table \ref{best_index_tb}}, we can see that NSG is a very sparse graph compared to other graph-based methods. Though the bottom layer of HNSW is sparse, it's much denser than NSG if we take the other layers into consideration. It's impossible to count the search path lengths for each method because the query points are not in the base data. Considering that all the graph-based method use the same search algorithm and most of the time is spent on distance calculations, the search performance can be a rough indicator of the term $ol$, where $o$ is the average out-degree and $l$ is the search path length. In \textbf{Figure \ref{graph_search}}, NSG outperforms the other graph-based methods on the four datasets. NSG has a lower $ol$ than other graph-based methods empirically. 

\item \textbf{Reduce the index size.} In \textbf{Table \ref{best_index_tb}}, NSG has the smallest indices on the four datasets. Especially, the index size of NSG is about 1/2-1/3 of the HNSW, which is the previous best performing algorithm\footnote{https://github.com/erikbern/ann-benchmarks}. It is important to note that, the memory occupations of NSG, HNSW, FANNG, Efanna's graph, and KGraph are all determined by the maximum out-degree. Although different nodes have different out-degrees, each node is allocated the same memory based on the maximum out-degree of the graphs to enable the continuous memory access (for better search performance). DPG cannot use this technique since its maximal out-degree is too large.

The small index of the NSG owes to approximating our MRNG and limit the maximal out-degree to a small value. The MRNG provides superior search complexity upper-bound for NSG. We have tried different auxiliary structures to replace the Navigating Node or use random starting node. The performance is not improved but gets worse sometimes. This means NSG approximates the MRNG well, and it does not need auxiliary structures for higher performance.
\end{enumerate}

\textbf{C. Some Interesting Points}:
\begin{enumerate}
\setlength{\itemsep}{1pt}
\setlength{\parskip}{0pt}
\setlength{\parsep}{0pt}
\item It's usually harder to search on datasets with higher local intrinsic dimension due to the ``curse of the dimensionality''. In \textbf{Figure \ref{graph_search}}, as the local intrinsic dimension increases, the performance gap between NSG and the other algorithms is widening.

\item When the required precision becomes high, the performance of many methods becomes even worse than that of the serial scan. NSG is tens of times faster than the serial scan at 99\% precision on SIFT1M and GIST1M. On RAND4M and GAUSS5M, all the algorithms have lower speed-up over the serial scan. NSG is still faster than the serial scan at 99\% precision.

\item The indexing of NSG is almost the fastest among graph-based methods but is much slower than the non-graph-based methods. Due to the space limit, we only list the preprocessing time of all the graph-based methods in \textbf{Table \ref{index_time}}. 

\item We count how many edges of the NNG are included in a given graph (\textbf{NN-percentage}) for all the compared graph-based methods, which are shown in \textbf{Table \ref{best_index_tb}}. We can see that HNSW and FANNG suffer from the same problem: a large proportion of edges between nearest neighbors are missing (\textbf{Table \ref{best_index_tb}}). It is because they initialize their graphs with random edges then refine the graphs iteratively. Ideally, they can link all the nearest neighbors when their indexing algorithms converge to optima, but they don't have any guarantee on its convergence, which will cause detour problems as we have discussed in Section \ref{MRNG}. This is one of the reasons that the search performance of FANNG is much worse than the NSG. Another reason is that FANNG is based on RNG, which is not monotonic. HNSW is the second best-performing algorithm because HNSW enables fast short-cuts via multi-layer graphs. However, it results in very large index size.


\item The difference between NSG-Naive and NSG is that NSG-Naive does not select Navigating Node and does not ensure the connectivity of the graph. Moreover, the probability to reserve the monotonic paths is smaller because its candidates for pruning only cover a small neighborhood. Though NSG-Naive uses the same edge selection strategy, the degree of the approximation is inferior to NSG, which leads to inferior performance. 

\item In the optimal index of KGraph and Efanna, the out-degrees are much larger than NSG. This is because the $k$NN graph used in KGraph and Efanna is an approximation of the Delaunay Graph. As discussed before, the Delaunay Graph is monotonic, which is almost fully connected on high dimensional datasets. When the $k$ of the $k$NN graph is sufficiently large, the monotonicity may be best approximated. However, the high degree damages the performance of KGraph and Efanna significantly.
\end{enumerate}

\subsubsection{Complexity And Parameters}
There are three parameters in the NSG indexing algorithm, $k$ for the $k$NN graph; $l$ and $m$ for Algorithm \ref{NSGbuild_alg}. In our experiments, we find that the optimal parameters will not change with the data scale. Therefore we tune the parameters by sampling a small subset from the base data and performing grid search for the optimal parameters.

We estimate the search and indexing complexity of the NSG on SIFT1M and GIST1M. Specifically, we first randomly sample subsets of different size from SIFT1M and GIST1M. Then we record the indexing time of NSG and the search time on the test set to reach 95\% precision for all these subsets. Next we draw the resulting curves (how the time increases with $N$) and use different functions to fit it. The estimated search complexity is about $O(n^{\frac{1}{d}}\log n^{\frac{1}{d}})$, and the complexity of Algorithm \ref{NSGbuild_alg} is about $O(n^{1+\frac{1}{d}}\log n^{\frac{1}{d}})$, where $d$ is approximately equal to the intrinsic dimension. This agrees with our theoretical analysis. We also estimate how the search complexity scales with $K$, the number of neighbors required. It's about $O(K^{0.46})$ or $O((logK)^{2.7})$.

Please see the appendix for the figures and the detailed analysis. 

\subsection{Search On DEEP100M}
The DEEP1B is a dataset containing one billion float vectors of 96 dimension released by Artem \etal \cite{babenko2016efficient}. We sample 100 million vectors from it and perform the experiments on a machine with i9-7980 CPU and 96GB memory. The dataset occupies 37 GB memory, which is the largest dataset that the NSG can process on this machine. We build the $k$NN graph with Faiss \cite{johnson2017billion} on four 1080Ti GPUs. The time to build the $k$NN graph is 6.75 hours and the time of Algorithm \ref{NSGbuild_alg} is 9.7 hours. The peak memory use of NSG at indexing stage is 92GB, and it is 55 GB at searching stage. We try to test HNSW, but it always triggers the Out-Of-Memory error no matter how we set the parameters. So we only compare NSG with Faiss. The result can be seen in \textbf{Figure \ref{deep1e}}.

\begin{figure}[t]
\begin{center}
\includegraphics[width=220pt]{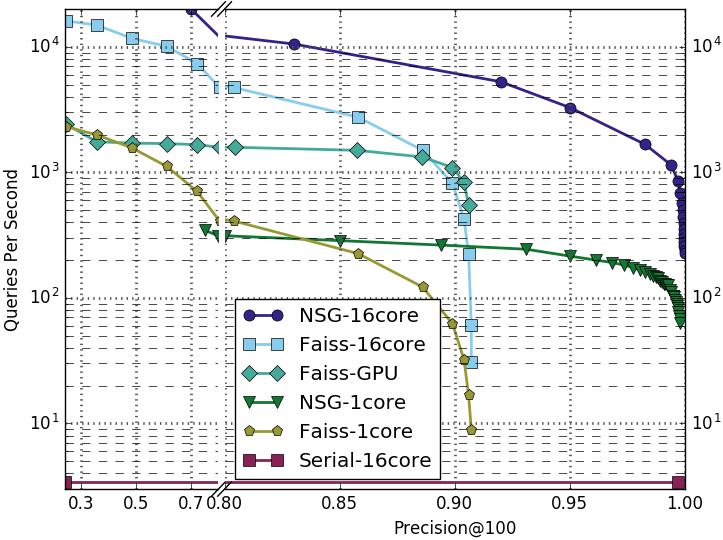}
\end{center}
   \caption{The ANNS performance of NSG and Faiss on the 100M subset of DEEP1B. Top right is better.}
\label{deep1e}
\end{figure}

\textbf{NSG-1core} means we build one NSG on the dataset and evaluate its performance with one CPU core. \textbf{NSG-16core} means we break the dataset into 16 subsets (6.25 million vectors each) randomly and build 16 NSG on these subsets respectively. In this way, we can enable inner-query parallel search for NSG by searching on 16 NSGs simultaneously and merging the results to get the final result. We build one Faiss index (IVFPQ) for the 100 million vectors and evaluate its performance with one CPU-core (\textbf{Faiss-1core}), 16 CPU-core (\textbf{Faiss-16 core}), and 1080Ti GPU (\textbf{Faiss-GPU}) respectively. Faiss supports inner-query parallel search. \textbf{Serial-16core} means we perform serial scan in parallel with 16 CPU cores. 

NSG outperforms Faiss significantly in high-precision region. Moreover, NSG-16core outperforms Faiss-GPU and is about 430 times faster than Serial-16core at 99\% precision. Meanwhile, building NSG on 6.25 million vectors takes 794 seconds. The total time of building 16 NSGs sequentially only spends 3.53 hours, which is much faster than building one NSG on the whole DEEP100M. The reason may be as follows. The complexity of Algorithm \ref{NSGbuild_alg} is about $O(n^{1+\frac{1}{d}}\log n^{\frac{1}{d}})$. Suppose we have a dataset $D$ with $n$ points. We can partition $D$ into $r$ subsets evenly. The time of building one NSG on $D$ is $t_1$. The time of building an NSG on one subset is $t_2$. It is easy to verify that we can have $t_1 > rt_2$ if we select a proper $r$. Consequently, sequentially indexing on subsets can be faster than indexing on the complete set.

\subsection{Search In E-commercial Scenario}
We have collaborated with Taobao on the billion-scale high-dimensional ANNS problem in the E-commercial scenario. The billion-scale data, daily updating, and response time limit are the main challenges. We evaluate NSG on the E-commercial data (128-dimension vectors of users and commodities) with different scales to work out a solution.

We compare NSG with the baseline (a well-optimized implementation of IVFPQ \cite{jegou2011product}) on the e-commerce database. We use a 10M dataset to test the performance on a single thread, and a 45M dataset to test the \textbf{Distributed Search} performance in a simulation environment. The simulation environment is a online scenario stress testing system based on MPI. We split the dataset and place the subsets on different machines. At search stage, we search each subset in parallel and merge the results to return. In our experiments, we randomly partition the dataset evenly into 12 subsets and build 12 NSGs. NSG is 5-10 times faster than the baseline at the same precision (See the appendix for details) and meet the response time requirement.

On the complete dataset (about 2 billion vectors), we find it impossible to build one NSG within one day. So we use the distributed search solution with 32 partitions. The average response time is about 5 ms at 98\% precision, and the indexing time is about 12 hours for a partition. The baseline cannot reach the response time requirement on the whole dataset.

\section{Discussions}
The NSG can achieve very high search performance at high precision, but it needs much more memory space and data-preprocessing time than many popular quantization-based and hashing-based methods (\eg, IVFPQ and LSH). The NSG is very suitable for high precision and fast response scenarios, given enough memory. In frequent updating scenarios, the indexing time is also important. Building one NSG on the large dataset is impractical, thus, the distributed search solution like our experiments may be a good choice. 

It's also possible for NSG to enable incremental indexing. We will leave this to future works. 

\section{Conclusions}
\label{sec5}
In this paper, we propose a new monotonic search network, MRNG, which ensures approximately logarithmic search complexity. We propose four aspects (ensuring the connectivity, lowering the average out-degree, shortening the search paths, and reducing the index size) to design better graph structure for massive problems. Based on the four aspects, we propose NSG, which is a practical approximation of the MRNG and considers the four aspects simultaneously. Extensive experiments show the NSG outperforms the other state-of-the-art algorithms significantly in different aspects. Moreover, the NSG outperforms the baseline method of Taobao (Alibaba Group) and has been integrated into their search engine for billion-scale search.

\section{Acknowledgements}
The authors would like to thank Tim Weninger and Cami G Carballo for their invaluable input in this work.


%






\bibliographystyle{ACM-Reference-Format}
\bibliography{nsg}  

\begin{thebibliography}{10}

\bibitem{andre2015cache}
F.~Andr{\'e}, A.-M. Kermarrec, and N.~Le~Scouarnec.
\newblock Cache locality is not enough: high-performance nearest neighbor
  search with product quantization fast scan.
\newblock {\em Proceedings of the VLDB Endowment}, 9(4):288--299, 2015.

\bibitem{AroraSK018}
A.~Arora, S.~Sinha, P.~Kumar, and A.~Bhattacharya.
\newblock Hd-index: Pushing the scalability-accuracy boundary for approximate
  knn search in high-dimensional spaces.
\newblock {\em {Proceedings of the VLDB Endowment}}, 11(8):906--919, 2018.


\bibitem{arya1993approximate}
S.~Arya and D.~M. Mount.
\newblock Approximate nearest neighbor queries in fixed dimensions.
\newblock In {\em Proceedings of the fourth annual ACM-SIAM Symposium on
  Discrete algorithms}, pages 271--280, 1993.

\bibitem{aurenhammer1991voronoi}
F.~Aurenhammer.
\newblock Voronoi diagrams—a survey of a fundamental geometric data
  structure.
\newblock {\em ACM Computing Surveys (CSUR)}, 23(3):345--405, 1991.

\bibitem{babenko2016efficient}
A.~Babenko and V.~Lempitsky.
\newblock Efficient indexing of billion-scale datasets of deep descriptors.
\newblock {\em Proceedings of the IEEE Conference on Computer Vision and
  Pattern Recognition}, pages 2055--2063, 2016.

\bibitem{beckmann1990r}
N.~Beckmann, H.-P. Kriegel, R.~Schneider, and B.~Seeger.
\newblock The r*-tree: an efficient and robust access method for points and
  rectangles.
\newblock In {\em Acm Sigmod Record}, volume~19, pages 322--331. Acm, 1990.

\bibitem{Ben2016Fanng}
H.~Ben and D.~Tom.
\newblock {FANNG}: Fast approximate nearest neighbour graphs.
\newblock In {\em Proceedings of the 2016 IEEE Conference on Computer Vision
  and Pattern Recognition}, pages 5713--5722, 2016.

\bibitem{Bentley1975Multidimensional}
J.~L. Bentley.
\newblock Multidimensional binary search trees used for associative searching.
\newblock {\em Communications of the {ACM}}, 18(9):509--517, 1975.

\bibitem{boguna2009navigability}
M.~Boguna, D.~Krioukov, and K.~C. Claffy.
\newblock Navigability of complex networks.
\newblock {\em Nature Physics}, 5(1):74--80, 2009.

\bibitem{chen2005robust}
L.~Chen, M.~T. {\"O}zsu, and V.~Oria.
\newblock Robust and fast similarity search for moving object trajectories.
\newblock In {\em Proceedings of the 2005 ACM SIGMOD international conference
  on Management of data}, pages 491--502. ACM, 2005.

\bibitem{costa2005estimating}
J.~A. Costa, A.~Girotra, and A.~Hero.
\newblock Estimating local intrinsic dimension with k-nearest neighbor graphs.
\newblock {\em Statistical Signal Processing, 2005 IEEE/SP 13th Workshop on},
  pages 417--422, 2005.

\bibitem{de2002efficient}
A.~P. de~Vries, N.~Mamoulis, N.~Nes, and M.~Kersten.
\newblock Efficient k-nn search on vertically decomposed data.
\newblock {\em Proceedings of the 2002 ACM SIGMOD international conference on
  Management of data}, pages 322--333, 2002.

\bibitem{dearholt1988monotonic}
D.~Dearholt, N.~Gonzales, and G.~Kurup.
\newblock Monotonic search networks for computer vision databases.
\newblock {\em Signals, Systems and Computers, 1988. Twenty-Second Asilomar
  Conference on}, 2:548--553, 1988.

\bibitem{Dong2011Efficient}
W.~Dong, C.~Moses, and K.~Li.
\newblock Efficient k-nearest neighbor graph construction for generic
  similarity measures.
\newblock {\em Proceedings of the 20th international Conference on World Wide
  Web}, pages 577--586, 2011.

\bibitem{CongEfanna2016}
C.~Fu and D.~Cai.
\newblock Efanna : An extremely fast approximate nearest neighbor search
  algorithm based on knn graph.
\newblock {\em arXiv:1609.07228}, 2016.

\bibitem{fu2017fast}
C.~Fu, C.~Xiang, C.~Wang, and D.~Cai.
\newblock Fast approximate nearest neighbor search with the navigating
  spreading-out graph.
\newblock {\em arXiv preprint arXiv:1707.00143}, 2018.

\bibitem{Fukunaga1975A}
K.~Fukunaga and P.~M. Narendra.
\newblock A branch and bound algorithm for computing k-nearest neighbors.
\newblock {\em IEEE Transactions on Computers}, 100(7):750--753, 1975.

\bibitem{gao2014dsh}
J.~Gao, H.~V. Jagadish, W.~Lu, and B.~C. Ooi.
\newblock Dsh: data sensitive hashing for high-dimensional k-nnsearch.
\newblock {\em Proceedings of the 2014 ACM SIGMOD international conference on
  Management of data}, pages 1127--1138, 2014.

\bibitem{ge2013optimized}
T.~Ge, K.~He, Q.~Ke, and J.~Sun.
\newblock Optimized product quantization for approximate nearest neighbor
  search.
\newblock {\em Proceedings of the IEEE Conference on Computer Vision and
  Pattern Recognition}, pages 2946--2953, 2013.

\bibitem{Gionis1999Similarity}
A.~Gionis, P.~Indyk, and R.~Motwani.
\newblock Similarity search in high dimensions via hashing.
\newblock {\em Proceedings of the International Conference on Very Large Data
  Bases}, pages 518--529, 1999.

\bibitem{Hajebi2011Fast}
K.~Hajebi, Y.~Abbasi-Yadkori, H.~Shahbazi, and H.~Zhang.
\newblock Fast approximate nearest-neighbor search with k-nearest neighbor
  graph.
\newblock {\em Proceedings of the International Joint Conference on Artificial
  Intelligence}, 22:1312--1317, 2011.

\bibitem{har2012approximate}
S.~Har{-}Peled, P.~Indyk, and R.~Motwani.
\newblock Approximate nearest neighbor: towards removing the curse of
  dimensionality.
\newblock {\em Theory of Computing}, 8(14):321--350, 2012.

\bibitem{HuangFZFN15}
Q.~Huang, J.~Feng, Y.~Zhang, Q.~Fang, and W.~Ng.
\newblock Query-aware locality-sensitive hashing for approximate nearest
  neighbor search.
\newblock {\em {Proceedings of the VLDB Endowment}}, 9(1):1--12, 2015.

\bibitem{jagadish2005idistance}
H.~V. Jagadish, B.~C. Ooi, K.-L. Tan, C.~Yu, and R.~Zhang.
\newblock {iDistance: An adaptive B+-tree based indexing method for nearest
  neighbor search}.
\newblock {\em ACM Transactions on Database Systems (TODS)}, 30(2):364--397,
  2005.

\bibitem{jaromczyk1992relative}
J.~W. Jaromczyk and G.~T. Toussaint.
\newblock Relative neighborhood graphs and their relatives.
\newblock {\em Proceedings of the IEEE}, 80(9):1502--1517, 1992.

\bibitem{jegou2011product}
H.~Jegou, M.~Douze, and C.~Schmid.
\newblock Product quantization for nearest neighbor search.
\newblock {\em IEEE transactions on pattern analysis and machine intelligence},
  33(1):117--128, 2011.

\bibitem{Jin2014Fast}
Z.~Jin, D.~Zhang, Y.~Hu, S.~Lin, D.~Cai, and X.~He.
\newblock Fast and accurate hashing via iterative nearest neighbors expansion.
\newblock {\em IEEE transactions on cybernetics}, 44(11):2167--2177, 2014.

\bibitem{johnson2017billion}
J.~Johnson, M.~Douze, and H.~J{\'e}gou.
\newblock Billion-scale similarity search with gpus.
\newblock {\em arXiv:1702.08734}, 2017.

\bibitem{kleinberg2000navigation}
J.~M. Kleinberg.
\newblock Navigation in a small world.
\newblock {\em Nature}, 406(6798):845--845, 2000.

\bibitem{kurup1992database}
G.~D. Kurup.
\newblock Database organized on the basis of similarities with applications in
  computer vision.
\newblock {\em Ph.D. thesis, New Mexico State University}, 1992.

\bibitem{li2016approximate}
W.~Li, Y.~Zhang, Y.~Sun, W.~Wang, W.~Zhang, and X.~Lin.
\newblock Approximate nearest neighbor search on high dimensional
  data---experiments, analyses, and improvement (v1. 0).
\newblock {\em arXiv:1610.02455}, 2016.

\bibitem{liu2014sk}
Y.~Liu, J.~Cui, Z.~Huang, H.~Li, and H.~T. Shen.
\newblock Sk-lsh: an efficient index structure for approximate nearest neighbor
  search.
\newblock {\em Proceedings of the VLDB Endowment}, 7(9):745--756, 2014.

\bibitem{malkov2014approximate}
Y.~Malkov, A.~Ponomarenko, A.~Logvinov, and V.~Krylov.
\newblock Approximate nearest neighbor algorithm based on navigable small world
  graphs.
\newblock {\em Information Systems}, 45:61--68, 2014.

\bibitem{MalkovYHNSW16}
Y.~A. Malkov and D.~A. Yashunin.
\newblock Efficient and robust approximate nearest neighbor search using
  hierarchical navigable small world graphs.
\newblock {\em arXiv:1603.09320}, 2016.

\bibitem{Manning2008Introduction}
C.~D. Manning, P.~Raghavan, H.~Sch{\"u}tze, et~al.
\newblock {\em Introduction to information retrieval}, volume~1.
\newblock Cambridge university press Cambridge, 2008.

\bibitem{Silpaanan2008Optimised}
C.~Silpa-Anan and R.~Hartley.
\newblock Optimised kd-trees for fast image descriptor matching.
\newblock {\em Proceedings of the IEEE Conference on Computer Vision and
  Pattern Recognition}, pages 1--8, 2008.

\bibitem{teodoro2014approximate}
G.~Teodoro, E.~Valle, N.~Mariano, R.~Torres, W.~Meira, and J.~H. Saltz.
\newblock Approximate similarity search for online multimedia services on
  distributed cpu--gpu platforms.
\newblock {\em The VLDB Journal}, 23(3):427--448, 2014.

\bibitem{toussaint1980relative}
G.~T. Toussaint.
\newblock The relative neighbourhood graph of a finite planar set.
\newblock {\em Pattern recognition}, 12(4):261--268, 1980.

\bibitem{weber1998quantitative}
R.~Weber, H.-J. Schek, and S.~Blott.
\newblock A quantitative analysis and performance study for similarity-search
  methods in high-dimensional spaces.
\newblock {\em Proceedings of the International Conference on Very Large Data
  Bases}, 98:194--205, 1998.

\bibitem{Weiss2008Spectral}
Y.~Weiss, A.~Torralba, and R.~Fergus.
\newblock Spectral hashing.
\newblock {\em Advances in neural information processing systems}, pages
  1753--1760, 2009.

\bibitem{wu2014fast}
Y.~Wu, R.~Jin, and X.~Zhang.
\newblock Fast and unified local search for random walk based
  k-nearest-neighbor query in large graphs.
\newblock {\em Proceedings of the 2014 ACM SIGMOD international conference on
  Management of data}, pages 1139--1150, 2014.

\bibitem{yao1982constructing}
A.~C.-C. Yao.
\newblock On constructing minimum spanning trees in k-dimensional spaces and
  related problems.
\newblock {\em SIAM Journal on Computing}, 11(4):721--736, 1982.

\bibitem{zheng2016lazylsh}
Y.~Zheng, Q.~Guo, A.~K. Tung, and S.~Wu.
\newblock Lazylsh: Approximate nearest neighbor search for multiple distance
  functions with a single index.
\newblock {\em Proceedings of the 2016 International Conference on Management
  of Data}, pages 2023--2037, 2016.

\end{thebibliography}

\clearpage
\begin{appendix}
\section{Proof For Theorem 1}
\begin{proof}
We can convert this problem into proving the following proposition. For any two points $p,q\in S$, we treat $q$ as the query and $p$ as the search starting point. For any $0<t<n$, after $t$ iterations of Algorithm \ref{search_alg} on the MSNET without back-tracing, the current subpath $v_1, ... , v_t, (v_1 = p)$ found by the algorithm is monotonic about $q$, which means $\forall{i}\in \{1,...,t-1\}, \delta(v_i, q) > \delta(v_{i+1}, q)$.

 If this proposition is true, then we can reach $q$ from $p$ after at most $n-1$ iterations of Algorithm \ref{search_alg}. Because after $n-1$ iterations, we will get a sequence $\{\delta(v_1,q),..., \delta(v_{n},q)\}$. Because the path $v_1, ... , v_{n}$ is monotonic about $q$, we have $\forall{i}\in \{1,...,n-1\}, \delta(v_i,q) > \delta(v_{i+1},q)$. If $\delta(v_{n},q) > 0$, we have $\forall{i}\in \{1,...,n\}, v_i \ne q$. Thus, we can conclude that $q \notin S$ and arrive at a contradiction. Therefore, we can reach $q$ in at most $n-1$ iterations. 
Now we will prove the new proposition with mathematical induction.

(1) Suppose $t=1$. There must exist at least one monotonic path between $v_1$ and $q$ in the MSNET. We select one monotonic path from $v_1$ to $q$, there must be a node $r$ on the monotonic path such that $r$ is the neighbor of $v_1$ and $\delta(v_1, q) > \delta(r, q)$. If $r = v_2$, then it's monotonic from $v_1$ to $v_2$ about $q$. Otherwise, because $v_2$ is found by Algorithm \ref{search_alg} without back-tracing, $v_2$ is the closest node to $q$ amongst $v_1$'s neighbors. Therefore, $\delta(v_2, q) \le \delta(r, q) < \delta(v_1,q)$. It's monotonic from $v_1$ to $v_2$ about $q$. 

(2) Suppose the proposition is true when $t=m$, the path $v_1, ... , v_{m}$ found by the algorithm without back-tracing is monotonic about $q$. When $t=m+1$, we select one monotonic path from $v_m$ to $q$, there must be a node $s$ on the monotonic path such that $s$ is the neighbor of $v_m$ and $\delta(v_m, q) > \delta(s, q)$. If $s = v_{m+1}$, then it's monotonic from $v_1$ to $v_{m+1}$ about $q$. Otherwise, because $v_{m+1}$ is found by Algorithm \ref{search_alg} without back-tracing, $v_{m+1}$ is the closest node to $q$ amongst $v_m$'s neighbors. Therefore, $\delta(v_{m+1}, q) \le \delta(s, q) < \delta(v_m,q)$. It's monotonic from $v_1$ to $v_{m+1}$ about $q$. 

Therefore, the new proposition is true, and we can find a monotonic path between any two nodes $p,q$ in $G$ with Algorithm \ref{search_alg} without back-tracing. 
\end{proof}

\section{Proof For Theorem 2}
\begin{proof}
(1) When $p$ and $q$ are linked in the MSNET, the path length between $p$ and $q$ is 1.

(2) When $p$ and $q$ are not linked in the MSNET, we can select a monotonic path in the MSNET, denoted as $p=v_1, v_2,...,v_{k-1}, q$. Thus, we have a monotonically decreasing sequence:
\begin{equation}
    Z = \{\delta(v_1,q), \delta(v_2,q),...,\delta(v_{k-1},q) \}
\end{equation}

Because $v_{k-1}$ may not be the nearest neighbor of $q$, we add a element $\delta(v_{k},q)$ to the end of the sequence, where $v_k$ is the nearest neighbor of $q$ and $\delta(v_{k-1},q) \ge \delta(v_k,q)$. The length of the new sequence still corresponds strictly to the length of the path (always longer by 1 for all possible paths), so it won't affect the correctness of the conclusion. We denote the new sequence as $Z = \{\delta(v_1,q), \delta(v_2,q),...,\delta(v_k,q) \}$.

We can use the elements in $Z$ to construct a sequence of concentric open spheres $\{B(q, \delta(v_1,q)), ... , B(q, \delta(v_{k}, q)) \}$. Let $V_B(B(q, \gamma))$ denote the volume of the sphere $B(q, \gamma)$.   Let
\begin{align*}
\eta_i &= \frac{V_B(B(q, \delta(v_{i+1},q)))}{ V_B(B(q, \delta(v_{i},q)))}, i \in \{1,..,k-2\}\\
\eta_i &= (\frac{\delta(v_{i+1},q)}{\delta(v_{i},q)})^d , i \in \{1,..,k-2\}
\end{align*}

Let $R$ be the maximal distance between any two points in $S$. We have $\forall{i} \in \{1,..,k\}, \delta(v_i,q) \le R$.

Suppose $a,b,c$ are any three points in $S$ such that triangle $abc$ is not an isosceles triangle. Let $\triangle{r}$=min$\{|\delta(a,b) - \delta(a,c)|, |\delta(a,b) - \delta(b,c)|, |\delta(a,c) - \delta(b,c)| \}$. We have $\forall{i} \in \{1,..,k-2\}$, $|\delta(v_i,q)-\delta(v_{i+1},q)| \ge \triangle{r}$.

Therefore, we have
\begin{align*}
\eta_i &= (\frac{\delta(v_{i+1},q)}{\delta(v_i,q)})^d \le (\frac{R-\triangle{r}}{R})^d, \forall{i} \in \{1,..,k-2\}\\
\eta_{k-1} &= (\frac{\delta(v_k,q)}{\delta(v_{k-1},q)})^d \le 1
\end{align*}

and 
\begin{align*}
V_{B_1} \le V_B(B(q, R)).
\end{align*}

Let $V_{B_i} = V_B(B(q, \delta(v_i,q)))$. We have
\begin{align*}
V_{B_{k}} &= V_{B_1} \frac{V_{B_2}}{V_{B_1}} ... \frac{V_{B_{k}}}{V_{B_{k-1}}} \\
                                            & = V_{B_1} \cdot \eta_1 ... ~ \eta_{k-1}
\end{align*}

Because $V_B(B(q, \delta(v_1,q))) \le V_B(B(q, R))$, we have
\begin{align*}
V_{B_{k}} &\le V_B(B(q, R)) \cdot \eta_1 ...~  \eta_{k-2}\eta_{k-1}\\
                                            &\le V_B(B(q, R)) \cdot ((\frac{R-\triangle{r}}{R})^d)^{k-2}\\
\frac{V_{B_{k}} }{V_B(B(q, R))}& \le ((\frac{R-\triangle{r}}{R})^d)^{k-2}
\end{align*}
Let $\hat{\eta} = (\frac{R-\triangle{r}}{R})^d$, we have
\begin{align*}
                                       k-2&\le  \log_{\hat{\eta}}\frac{V_{B_{k}}}{V_B(B(q, R))}\\
                                            & = \log_{\hat{\eta}}(\frac{\delta(v_k, q)}{R})^d\\
                                            & = d\log_{\hat{\eta}}\delta(v_k, q) - d\log_{\hat{\eta}}R
\end{align*}
We calculate the expectation of above inequality and have
\begin{align*}
\mathbb{E}(k) -2\le & d\mathbb{E}( \log_{\hat{\eta}}(\delta(v_k, q) )) - d\mathbb{E}(\log_{\hat{\eta}}( R) )
\end{align*}

Because $R$ and $ \triangle{r}$ are constants (independent of the choice of $p,q$) when the point set $S$ and the MSNET are determined, we have $\hat{\eta} = (\frac{R-\triangle{r}}{R})^d$ to be independent of the choice of $p,q$. Therefore, we have
\begin{align*}
\mathbb{E}(k) - 2\le & d\mathbb{E}( \log_{\hat{\eta}}(\delta(v_k, q) )) - d\mathbb{E}(\log_{\hat{\eta}}( R) )\\
		           =  & \frac{d\mathbb{E}(\log \delta(v_k, q))}{\log\hat{\eta}} - \frac{d\log R }{\log \hat{\eta}}
\end{align*}

Because $\log\mathbb{E}(x) \ge \mathbb{E}(\log x)$, we have
\begin{align*}
\mathbb{E}(k) - 2 \le & \frac{d\log \mathbb{E}(\delta(v_k, q))}{\log\hat{\eta}} - \frac{d\log R}{\log \hat{\eta}}\\
			   = & \frac{d\log \mathbb{E}(\delta(v_k, q)) - d\log R}{d\log (R - \triangle{r}) - d\log R}\\
			   = & \frac{\log \mathbb{E}(\delta(v_k, q)) - \log R}{\log (R - \triangle{r}) - \log R}
\end{align*}
Let
\begin{align*}
			f(R)  = & \frac{\log \mathbb{E}(\delta(v_k, q))) - \log R}{\log (R - \triangle{r}) - \log R}
\end{align*}

Because $\mathbb{E}(\delta(v_k, q)))$ and $\triangle{r}$ is independent of $R$, we can calculate the derivative of $f(R)$ about $R$ and get
\begin{align*}
			f'(R)   = & \frac{\frac{1}{R}(\log R - \log(R- \triangle{r})) }{(\log (R - \triangle{r}) - \log R) ^ 2}\\
				 + & \frac{ (\log R - \log(\mathbb{E}(\delta(v_k, q))))(\frac{1}{R - \triangle{r}} - \frac{1}{R})}{(\log (R - \triangle{r}) - \log R) ^ 2}
\end{align*}

By the definition of $\triangle{r}$ and $R$, we have $\triangle{r} < R$. Because $\delta(v_k, q)$ is the distance between $q$ and $v_k$, which is the nearest neighbor of $q$ in $S$, we have $\forall{q} \in S, R \ge \delta(v_k, q)$. Further, $R \ge \mathbb{E}(\delta(v_k, q))$. Therefore, we have $\forall{R}>0, f'(R) >0$. 

We define the density of the data points to be the number of the points in a unit volume. Because the data points are uniformly distributed in the space, the density $A$ of the data points is a constant. We have $n = A \cdot V_S$. Because $\kappa V_S\ge V_B(B(q, R))$, we have
\begin{align*}
 V_B(B(q, R)) & \le \kappa V_S\\
 V_B(B(q, R)) & \le \kappa \frac{n}{A}\\
 C_BR^d & \le \kappa \frac{n}{A}\\
 R & \le (\frac{\kappa n}{AC_B})^\frac{1}{d},
\end{align*}
where $C_B$ is a constant that is only related to $d$. For convenience, we merge the constant factors and rewrite above inequality as $ R  \le (\frac{n}{C})^\frac{1}{d}$. Because $\forall{R}>0, f'(R) >0$, we have $f(R) \le f((\frac{n}{C})^\frac{1}{d})$, \textit{i.e.}, we have
\begin{align*}
\mathbb{E}(k) - 2 &\le f(R) \le f((\frac{n}{C})^\frac{1}{d})\\
			   &= \frac{\log \mathbb{E}(\delta(v_k, q)) - \log (\frac{n}{C})^\frac{1}{d}}{\log ((\frac{n}{C})^\frac{1}{d} - \triangle{r}) - \log (\frac{n}{C})^\frac{1}{d}}
\end{align*}

Now we are going to evaluate growth rate of the right side of the above inequality about $n$. 

Because $\delta(v_k, q)$ is the distance between $q$ and its nearest neighbor. The expectation $\mathbb{E}(\delta(v_k, q))$ over all possible $q$ is only related to the density of the data points, thus, it's a constant which is independent of $n$. Let $l = \log \mathbb{E}(\delta(v_k, q))$ and $r(n) = (\frac{n}{C})^\frac{1}{d}$, then we have
\begin{align*}
\mathbb{E}(k) - 2 \le & \frac{l - \log r(n)}{\log (r(n) - \triangle{r}) - \log r(n)}\\
\end{align*}
Let
\begin{align*}
f_1(n) = & \frac{l - \log r(n)}{\log (r(n) - \triangle{r}) - \log r(n)}\\
f_2(n) = & \frac{r(n)\log r(n)}{\triangle{r}}
\end{align*}
We are going to prove that $f_1(n)$ and $f_2(n)$ are of the same order of growth. 

Suppose $S_n=\{p_1, ... ,p_n\}$. By the definition of $\triangle{r}$, we can find that, when $n$ increases by 1, there will be a new point $p_{n+1}$ added to $S$. And there will be $C_n^2$ triangles that can be formed by $p_i, p_j \in S_n$ and $p_{n+1}$. Amongst the $C_n^2$ triangles, it's possible to obtain a $\triangle{r}'$ which is smaller than the $\triangle{r}$ of the original point set $S_n$. Otherwise, the $\triangle{r}$ on the new set $S_{n+1}$ will remain unchanged. Thus, $\triangle{r}$ is a monotonically decreasing function of $n$, and when $n \rightarrow \infty$, $\triangle{r} \rightarrow 0$. Because $r(n) = (\frac{n}{C})^\frac{1}{d}$, we have $r(n) \rightarrow \infty$ when $n \rightarrow \infty$.

Because $n$ is discontinuous ($n \in \mathbb{N}^+$), for the convenience of the discussion of limit, we assume $n$ is a continuous variable. And we assume $r(n)$ and $\triangle{r}$ are continuously differentiable about $n$, which won't influence the correctness of the results. 

We have
\begin{align*}
\frac{f_1(n)}{f_2(n)} &= \frac{\frac{l - \log r(n)}{\log (r(n) - \triangle{r}) - \log r(n)}}{\frac{r(n)\log r(n)}{\triangle{r}}}\\
			      &= \frac{\frac{l}{\log r(n)}-1}{\frac{r(n)(\log (r(n) - \triangle{r}) - \log r(n))}{\triangle{r}}}
\end{align*}
Because when $n\rightarrow \infty$, we have 
\begin{align*}
\frac{l}{\log r(n)}-1 \rightarrow & -1 
\end{align*}
Let 
\begin{align*}
f_3(n) = \frac{r(n)(\log (r(n) - \triangle{r}) - \log r(n))}{\triangle{r}}
\end{align*}

Because $\underset{n\rightarrow \infty}{lim}f_3(n)$ is not obvious, we rewrite $f_3(n)$ as follows:
\begin{align*}
f_3(n) = \frac{\log (r(n) - \triangle{r}) - \log r(n)}{\frac{\triangle{r}}{r(n)}}
\end{align*}

Let
\begin{align*}
f_4(n) = &\log (r(n) - \triangle{r}) - \log r(n),\\
f_5(n) = &\frac{\triangle{r}}{r(n)}
\end{align*}
We have $\underset{n\rightarrow \infty}{lim}f_4(n) = 0$ and $\underset{n\rightarrow \infty}{lim}f_5(n) = 0$. According to L'H$\hat{o}$pital's rule, we have
\begin{align*}
\underset{n\rightarrow \infty}{lim}f_3(n)& =\underset{n\rightarrow \infty}{lim}\frac{f_4(n)}{f_5(n)} = \underset{n\rightarrow \infty}{lim}\frac{f_4'(n)}{f_5'(n)}\\
									       & =  \underset{n\rightarrow \infty}{lim}\frac{\frac{r'(n) - \triangle{r}'}{r(n) - \triangle{r}} - \frac{r'(n)}{r(n)}}{\frac{r(n)\triangle{r}' - r'(n)\triangle{r}}{r^2(n)}}\\
									       & = \underset{n\rightarrow \infty}{lim}\frac{\frac{r'(n)\triangle{r}-r(n)\triangle{r}'}{(r(n) - \triangle{r})r(n)}}{\frac{r(n)\triangle{r}' - r'(n)\triangle{r}}{r^2(n)}}\\
									       & = \underset{n\rightarrow \infty}{lim}\frac{-r^2(n)}{(r(n) - \triangle{r})r(n)}\\
									       & = \underset{n\rightarrow \infty}{lim}\frac{-1}{1 - \frac{\triangle{r}}{r(n)}}\\
									       & = -1
\end{align*}
Further, we have
\begin{align*}
\underset{n\rightarrow \infty}{lim}\frac{f_1(n)}{f_2(n)} &= \underset{n\rightarrow \infty}{lim}\frac{\frac{l}{\log r(n)}-1}{f_3(n)}\\
									        &= \frac{-1}{-1}\\
									        &= 1
\end{align*}

Therefore, we have proved that $f_1(n)$ and $f_2(n)$ are of the same order of growth. We have
\begin{align*}
\mathbb{E}(k) - 2 &= O(f_2(n))\\
			   &= O(\frac{n^{\frac{1}{d}}\log n^{\frac{1}{d}}}{\triangle{r}})
\end{align*}

We have the expectation length of a monotonic path in the given MSNET is $\mathbb{E}(length_{path}) = \mathbb{E}(k-1) = O(n^{\frac{1}{d}}\log n^{\frac{1}{d}}/\triangle{r})$.

\end{proof}

\section{Proof For Theorem 3}
\begin{proof}
Suppose $p,q \in S$ are any two points. If $\overset{\longrightarrow}{pq} \in$ \textit{MRNG}, then there is a monotonic path from $p$ to $q$ and we are done. Otherwise, according to the definition of the MRNG, there must exist a node $r$ such that $r \in lune_{pq}, \overset{\longrightarrow}{pr} \in$ \textit{MRNG}. Therefore we have $\delta(p,q) >\delta(r,q)$. Because for any two nodes $p,q$, there is at least one edge $\overset{\longrightarrow}{pr}$ such that $r \in B(q, \delta(p,q))$ in an MRNG, we can conclude that MRNG is an MSNET by Lemma 1.
\end{proof}

\section{Proof For Lemma 1}
\begin{proof}
(1)Sufficiency. If edge $\overset{\longrightarrow}{pq} \in G$, we have $\overset{\longrightarrow}{pq} \in B(q, \delta(p,q))$. If not, because $G$ is an MSNET, for any two nodes $p,q$, there exists a monotonic path from $p$ to $q$, denoted as $v_1,...,v_k, (k>2, v_1=p, v_k=q)$. We have $\delta(v_1,q) > \delta(v_2,q)$. Therefore, $v_2 \in B(q, \delta(v_1,q))$. Because $v_2$ is connected to $p(v_1)$, the proposition is true. 

(2)Necessity. If edge $\overset{\longrightarrow}{pq} \in G$, there is a monotonic path between $p,q$ with only one edge. Otherwise, we can prove that there exists at least one monotonic path from $p$ to $q$ with mathematical induction. This problem can be converted to that there must be a path $v_1, v_2, ... , v_k, (k\le n, v_1=p)$ such that the path $v_1, ... , v_k$ is monotonic about $q$, which means $\forall{i}\in \{1,...,k-1\}, \delta(v_i, q) > \delta(v_{i+1}, q)$.

When $k=2$, because there is at least one edge $\overset{\longrightarrow}{v_1v_2}$ such that $v_2 \in B(q, \delta(v_1,q))$, we have $\delta(v_2,q) < \delta(v_1,q)$.Thus, the path $v_1,v_2$ is monotonic about $q$. 

Suppose the proposition is true when $k=m, m < n$. When $k = m+1$, because there is at least one edge $\overset{\longrightarrow}{v_mv_{m+1}}$ such that $v_{m+1} \in B(v_m, \delta(v_m,q))$, we have $\delta(v_{m+1},q) < \delta(v_m,q)$. Therefore, when $k=m+1$, the path $v_1, ... , v_{m+1}$ is monotonic about $q$. 

Here we have proved that the path $v_1, .. , v_k, (p=v_1)$ is monotonic about $q$. Similar to the proof in Theorem \ref{MSNET_find}, we can finally reach $q$ for some $k\le n$.  Therefore, $G$ is an MSNET. 
\end{proof}

\section{Proof For Lemma 2}
\begin{proof}
Yao\cite{yao1982constructing} prove that, for any $0 < \varphi < \pi$, one can cover the space $E^d$ with finite convex cones sharing the same apex such that the angular diameter of each cone is smaller than $\varphi$. Let $\mathbb{C} = \{C_1, C_2, ..., C_k\}$ be such a set of convex cones constructed by the algorithm proposed by Yao. We have $sup\{\Theta(C_i)| C_i \in \mathbb{C}\}< \varphi$, where $\Theta(C)$ denotes the angular diameter of convex cone $C$. Let $u$ be the apex of all the cones and $ua$ be a ray. We define two convex cones $C_i, C_j$ to be \textit{adjacent} if $\exists{ua}, ua \in (C_i \cap C_j)$.

Let $\{uv_1, uv_2, ... ,uv_{k+1} \}$ be any $k+1$ rays with the same initial point $u$. Because the $k$ convex cones in $\mathbb{C}$ cover the whole $E^d$ space, according to the Pigeonhole Principle, we have that there are at least two rays are in the same cone. Let $uv$ and $uw$ be any two rays with the same initial point. Let $\angle{vuw}$ be the angle between $uv$ and $uw$. Because $sup\{\Theta(C_i)| C_i \in \mathbb{C}\} < \varphi$, we have, if $uv$ and $uw$ are inside of the same cone, then $\angle{vuw} \le \varphi$. Let $\varphi < 60^\circ$ and we construct a finite set of cones covering the whole $E^d$. Suppose the number of the cones is $K_d$. If we want the angle between any two rays sharing the same initial point to be no smaller than $60^\circ$, we have that the number of the rays must be smaller than $K_d+1$, where $K_d$ is only related to $d$.

Because $S$ is a finite point set, the direction to which the out-edges of each node pointing is discontinuous. Because the angles between any two out-edges of a given node is no smaller than $60^\circ$ in the MRNG, the max out-degree of each node is much smaller than $K_d+1$. Let $C_d$ be the maximal value of the average out-degree of any MRNG, we have $C_d < K_d+1$. $C_d$ is only related to the dimension $d$ and independent of $n$. 
\end{proof}

\section{Non-Graph-Based v.s. Graph-Based}
\begin{figure}[t]
\begin{center}
\includegraphics[width=240pt]{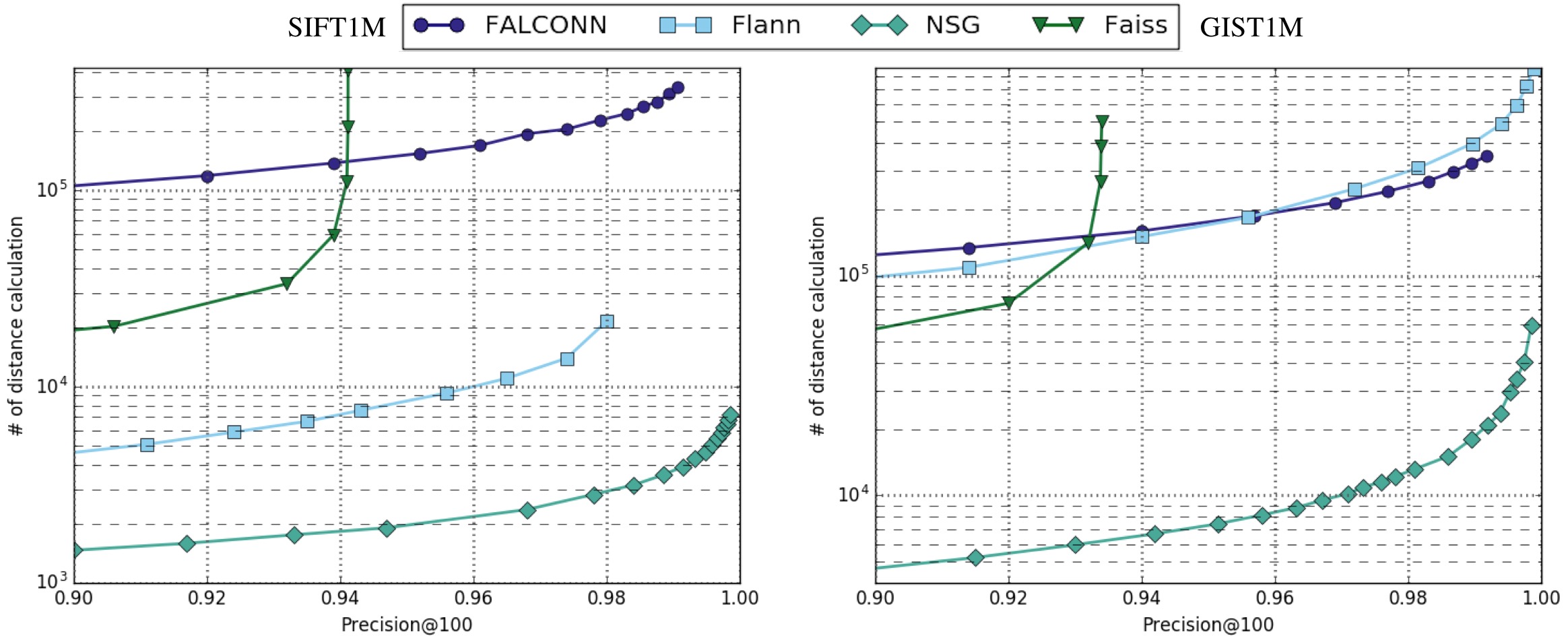}
\end{center}
   \caption{The numbers of distance calculations are recorded for the four algorithms. Bottom right is better}
\label{sift_gist_num}
\end{figure}
We choose some typical hashing-based method (LSH), tree-based method (Randomized KD-trees), quantization-based method (IVFPQ), and graph-based method (NSG). For the compared algorithms, we use FALCONN (LSH), Flann (Randomized KD-trees), and Faiss (IVFPQ) to get their performance. We count the number of distance calculations they use to reach a given precision and draw the curves in \textbf{Figure \ref{sift_gist_num}}. The results show that NSG needs tens of times less distance calculations than the others. This is the main reason of its high performance. 

\section{Experiments On The Connectivity}
\begin{table}[t]\scriptsize
	\caption{The number of strongly connected components (SCC) of all the graph-based methods. Because the search of the NSG and HNSW starts from fixed nodes, the SCC of them is recorded as 1 if all the other points are reachable from the search starting node.}
	\label{graph_scc}
	\centering
	\begin{tabular}{|p{1.2cm}<{\centering}|p{1.1cm}<{\centering}|p{1.5cm}<{\centering}|p{1.1cm}<{\centering}|p{1.3cm}<{\centering}|}
		\hline
		dataset &  algorithm & SCC amount &algorithm &SCC amount\\
		\hline
		\multirow{3}{*}{SIFT1M} 
		&NSG & 1 & HNSW & 1 \\
		\cline{2-5}
		&FANNG & 1 & DPG & 1 \\
		\cline{2-5}
		&KGraph & 1 & Efanna & 2 \\
		\hline
		\multirow{3}{*}{GIST1M} 
		&NSG & 1 & HNSW & 1\\
		\cline{2-5}
		&FANNG & 5 & DPG & 6 \\
		\cline{2-5}
		&KGraph & 5 & Efanna & 5 \\
		\hline
		\multirow{3}{*}{RAND4M} 
		&NSG & 1 & HNSW & 1\\
		\cline{2-5}
		&FANNG & 30 & DPG & 33 \\
		\cline{2-5}
		&KGraph & 27 & Efanna & 27 \\
		\hline
		\multirow{3}{*}{GAUSS5M} 
		&NSG & 1 & HNSW & 1\\
		\cline{2-5}
		&FANNG & 44 & DPG & 41 \\
		\cline{2-5}
		&KGraph & 36 & Efanna & 36 \\
		\hline
	\end{tabular}
\end{table}

The experimental results of graph connectivity are shown in \textbf{Table \ref{graph_scc}}. Algorithms like FANNG, KGraph, and DPG start their search from a randomly selected node. Efanna starts their search from a node given by searching in the randomized KD-trees. For these algorithms, the search starting node varies with the queries. To ensure the connectivity, they need to guarantee that there exists at least one path between any two pair of nodes, \ie, to guarantee their graph to be strongly connected. HNSW and our NSG start the search from one fixed node. So we just need to ensure there exists at least one path from this node to any other node. 

For these graph-based methods, we count the strongly connected components (SCC) in their graphs. For HNSW and NSG, we record their SCCs as one if the connectivity from the search starting node to all the other nodes is achieved. We find that only HNSW and NSG guarantee the connectivity of their graphs.

\section{Experiments On Taobao Data}

\begin{table}[t]\scriptsize
\caption{Results on the e-commerce dataset. E10M has 10 million vectors, E45M has 45 million vectors, E2B has 2 billion vectors. The dimension is 128. NT is the number of the threads. SQR98 means Single-Query-Response time to retrieve 100 neighbors at 98\% precision. IVFPQ is the quantization-based baseline. The baseline cannot reach the response time requirement on E2B, so we didn't report the time of it on E2B.}
\label{e-commerce}
\centering
\begin{tabular}{|c|c|c|c|}
\hline
data set & algorithm  & NT & SQR98 (ms)\\
\hline
E10M &  NSG & 1 & 2.3 \\
E10M &  IVFPQ & 1 & 10 \\
E45M &  NSG & 12 & 1 \\
E45M &  IVFPQ & 12 & 10 \\
E2B &  NSG & 32 & 5 \\
\hline
\end{tabular}
\end{table}
The experimental results of NSG on the Taobao Data are shown in \textbf{Table \ref{e-commerce}}. In real scenario, we need to process huge amount of data. Meanwhile, the data is daily updated. We need to complete the indexing within a day. Therefore we test the NSG on different scale of the datasets to work out a solution. The original search engine is an implementation of IVFPQ. We choose it as the baseline. Though the baseline can reach the response time requirement on E10M and E45M, NSG is 5-10 times faster than it at the same precision in our experiments. Because the baseline cannot reach the response time requirement on the whole dataset, we cannot test it online and don't report its search time here. 

The main challenge of using NSG is the large indexing time. We turn to the distributed search, building multiple NSGs on different subsets and putting them on different machines. In our implementation, we use Faiss to build the $k$NN graph on GPUs and perform Algorithm 2 on CPUs. At search stage, we search on all the machines in parallel and merge the results. The average response time is about 5ms at 98\% precision. The indexing time of one partition is about 12 hours, which makes daily updating possible. 

\section{Experiments On Search And Indexing Complexity}
We estimate the search and indexing complexity of the NSG on SIFT1M and GIST1M to verify our theoretical analysis. NSG's search complexity on SIFT1M is $O(n^{\frac{1}{9.3}}\log n^{\frac{1}{9.3}})$ for both 1-NN (\textbf{Figure \ref{1nn_search}}) and 100-NN search (\textbf{Figure \ref{100nn_search}}). NSG's search complexity on GIST1M is $O(n^{\frac{2}{18.9}}\log n^{\frac{1}{18.9}})$ for both 1-NN (\textbf{Figure \ref{1nn_search}}) and 100-NN search (\textbf{Figure \ref{100nn_search}}). 9.3 and 18.9 is close to the local intrinsic dimension of SIFT1M and GIST1M respectively (in Table 2 of the paper). The complexity of the search on the MRNG is $O(cn^{\frac{1}{d}}\log n^{\frac{1}{d}}/\triangle{r})$, where $c$ is the average degree of the MRNG. The search complexity of the NSG is quite close to that of MRNG.  

According to the results, we find that the $\triangle{r}$ on SIFT1M is almost a constant and has little impact on the search complexity. The $\triangle{r}$ on GIST1M is about $O(n^{-\frac{1}{18.9}})$. This may be mainly due to the difference in their data distribution. The numerical values in each dimension of SIFT vectors are integers ranging from 0 to 255, while the numerical values in each dimension of GIST vectors are real numbers ranging from 0 to 1.5. The density of the points in GIST1M is much larger than in SIFT1M. Nevertheless, $O(n^{-\frac{1}{18.9}})$ decreases very slow when $n$ increases on GIST1M. The search complexity increases very slowly with $n$. These results agree with our theoretical analysis, and we can see that NSG is a good approximation of the MRNG.

The indexing of the NSG includes the approximate $k$NN graph construction, the ``search-collect-select'' procedure and the tree spanning procedure. The approximate $k$NN graph construction method is replaceable when there is a more efficient approach. In this paper, we use the Faiss library and $nn$-descent algorithm. The empirical complexity of $nn$-descent is roughly $O(n^{1.14})$ \cite{Dong2011Efficient}. The complexity of Faiss is not given by the authors. The tree spanning procedure is usually very fast and is of about $O(n)$ complexity because it is simply a DFS on the graph and it is rare that a node is not connected to the tree unless the data points are highly clustered. Usually, we just scan the dataset one time when spanning the tree. The reconnecting operation is seldom needed. 

The ``search-collect-select'' is the most time-consuming part. We first search on the prebuilt $k$NN graph, which is an approximation of the Delaunay Graph. The Delaunay Graph is an MSNET. The search on the $k$NN graph can be estimated approximately through the search complexity on an MSNET. Because we search for all the nodes in the dataset and select among the nodes on the search path, the complexity of the ``search-collect-select'' should be about $O(\epsilon n^{1+\frac{1}{d}}\log n^{\frac{1}{d}})$ in total, where $\epsilon$ is some constant.

The experimental results are in \textbf{Figure \ref{index_complexity}}. We can see that $O(\epsilon n^{1+\frac{1}{d}}\log n^{\frac{1}{d}})$ fit the indexing time curves well on both datasets, which agrees with our theoretical analysis. 

We also estimate how the search complexity scales with $K$, the number of neighbors required. Because there is no theoretical analysis on how the search complexity may scale with $K$. We try to fit the figure with $O(K^x)$ and $O(\log^x K)$ and to find out the value of $x$. We find that $O(K^{0.46})$ or $O((logK)^{2.7})$ all fit the curves well (\textbf{Figure \ref{knn_search}}). The difference is the constant factors hidden by the big $O$ notation. 

\section{Parameter Tuning}
In our experimental study, there exists an optimal parameter setting for a given dataset. Interestingly, the setting doesn't change much with the data scale. The reason may be that the optimal parameters of NSG are determined by the data distribution. The distribution of a randomly sampled subset is very similar to the whole dataset. For time-saving, we can randomly sample a subset from the large dataset and test different parameter combinations in a grid-search manner. Specifically, in the parameter space, we set a search range for each dimension (parameter). The value of that dimension is tested from low to high with a preset step size. We will try all possible parameter combinations within the given ranges in a brute-force manner. And we will choose the parameter combination which produces the best search performance.

\clearpage

\begin{figure*}[t]
	\centering
	\subfigure[SIFT1M]{\includegraphics[width=200pt]{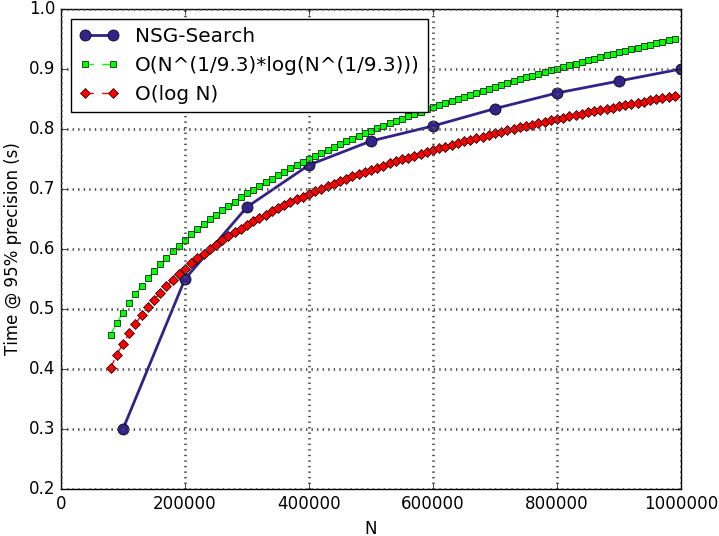}}
	\subfigure[GIST1M]{\includegraphics[width=200pt]{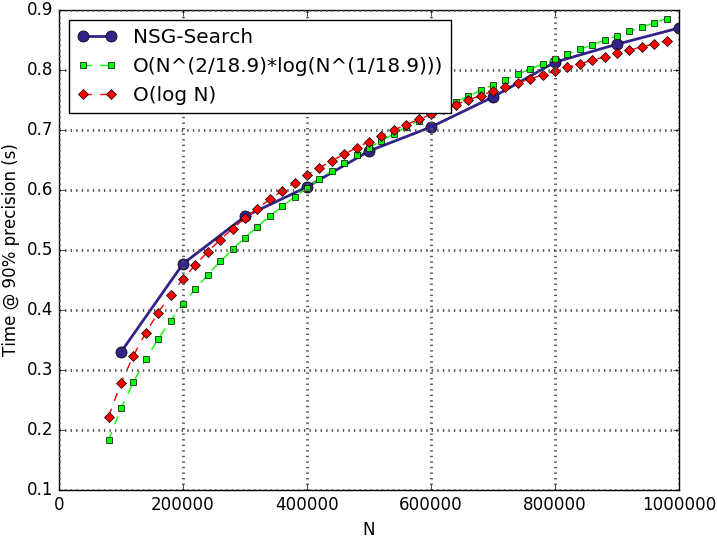}}
	\caption{The experiments of how the 1-NN search time scales with the data size on SIFT1M and GIST1M. The search time is recorded at 95\% precision and 90\% precision on them respectively. The complexity of 1-NN search on NSG is about $O(N^{\frac{1}{9.3}}\log N^{\frac{1}{9.3}})$ on SIFT1M, and is a little higher than $O(\log N)$. 9.3 is close to the intrinsic dimension of SIFT1M. The complexity of 1-NN search on NSG is about $O(N^{\frac{2}{18.9}}\log N^{\frac{1}{18.9}})$ on GIST1M, and is also a little higher than $O(\log N)$. 18.9 is close to the intrinsic dimension of GIST1M.}
	\label{1nn_search}
\end{figure*}

\begin{figure*}[t]
	\centering
	\subfigure[SIFT1M]{\includegraphics[width=200pt]{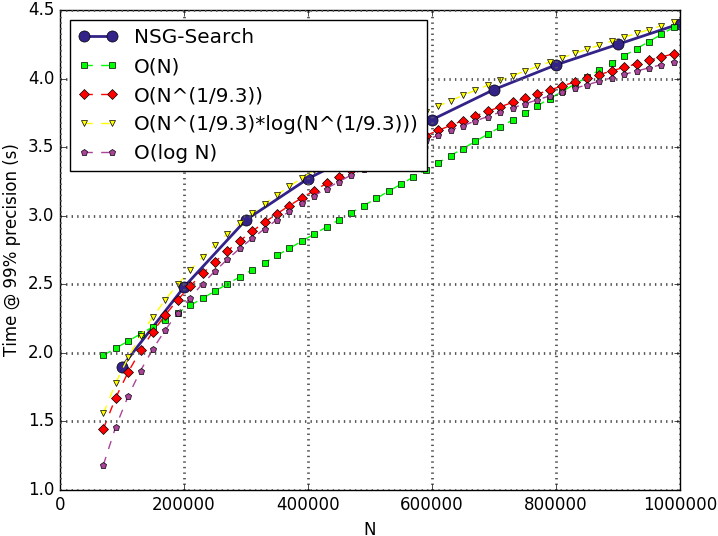}}
	\subfigure[GIST1M]{\includegraphics[width=200pt]{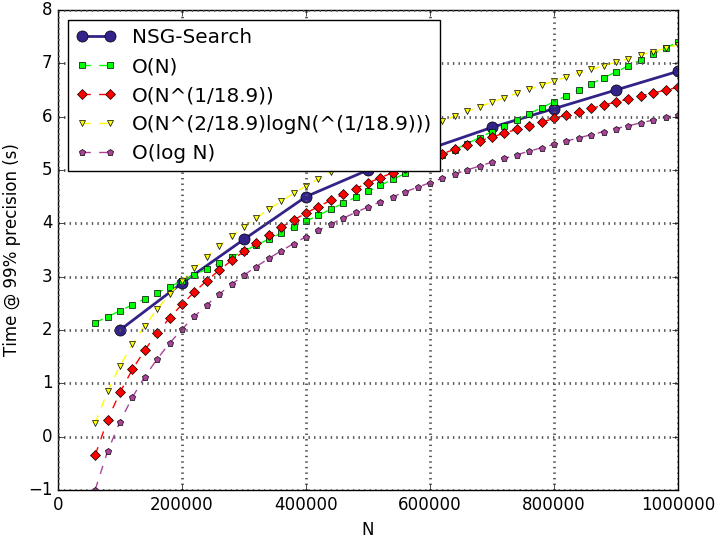}}
	\caption{The experiments of how the 100-NN search time scales with the data size on SIFT1M and GIST1M. The search time is recorded at 99\% precision. The complexity of 100-NN search on NSG is about $O(N^{\frac{1}{9.3}}\log N^{\frac{1}{9.3}})$ on SIFT1M, and is a little higher than $O(\log N)$ and $O(N^{\frac{1}{9.3}})$. 9.3 is close to the intrinsic dimension of SIFT1M. The complexity of 100-NN search on NSG is about $O(N^{\frac{2}{18.9}}\log N^{\frac{1}{18.9}})$ on GIST1M, and is a little higher than $O(\log N)$ and $O(N^{\frac{1}{18.9}})$. 18.9 is close to the intrinsic dimension of GIST1M.}
	\label{100nn_search}
\end{figure*}

\begin{figure*}[t]
	\centering
	\subfigure[SIFT1M]{\includegraphics[width=200pt]{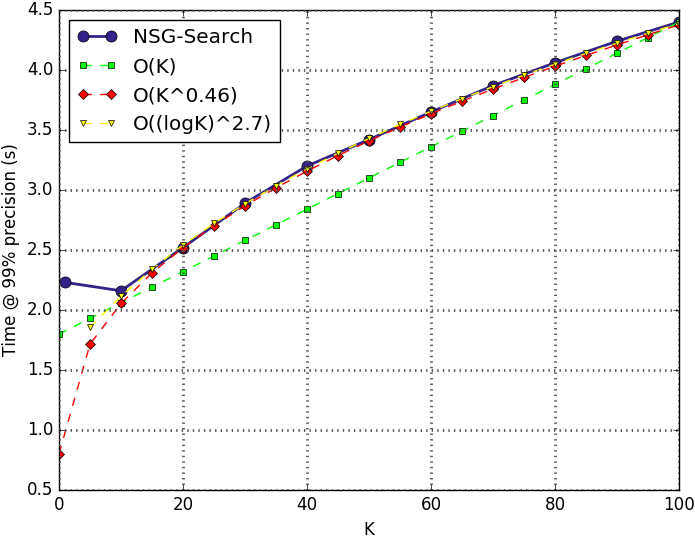}}
	\subfigure[GIST1M]{\includegraphics[width=200pt]{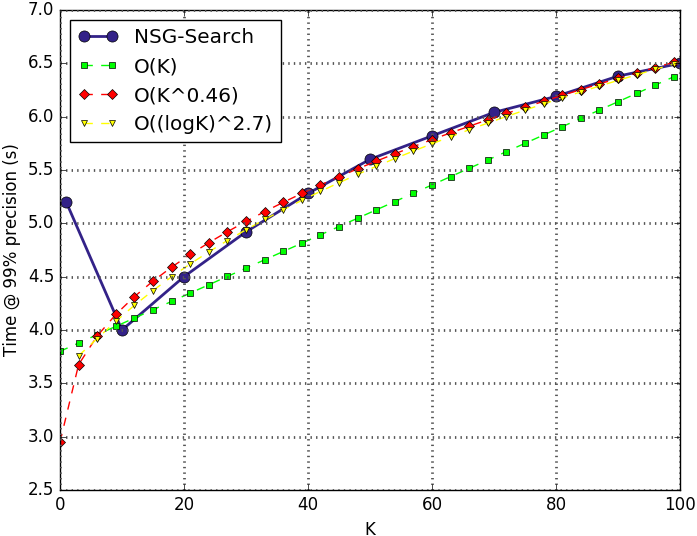}}
	\caption{The experiments of how the K-NN search time scales with K on SIFT1M and GIST1M. The search time is recorded at 99\% precision. The complexity of K-NN search on NSG is about $O(K^{0.46})$ or $O(\log^{2.7} K)$ on both datasets.}
	\label{knn_search}
\end{figure*}

\begin{figure*}[t]
	\centering
	\subfigure[SIFT1M]{\includegraphics[width=200pt]{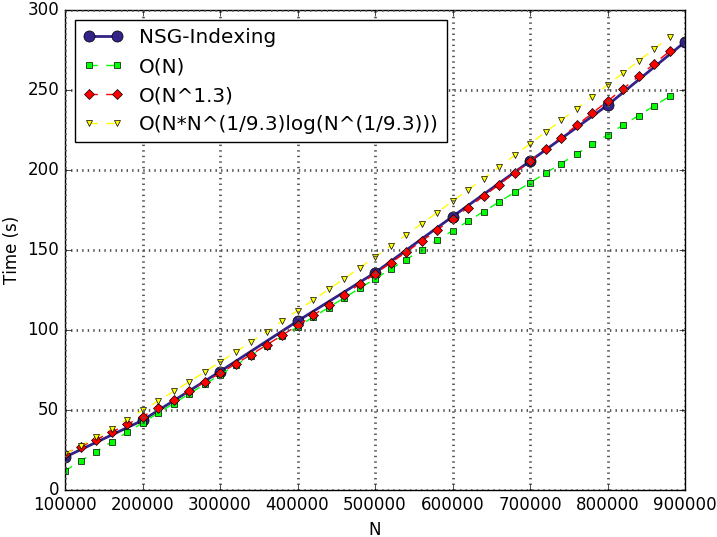}}
	\subfigure[GIST1M]{\includegraphics[width=200pt]{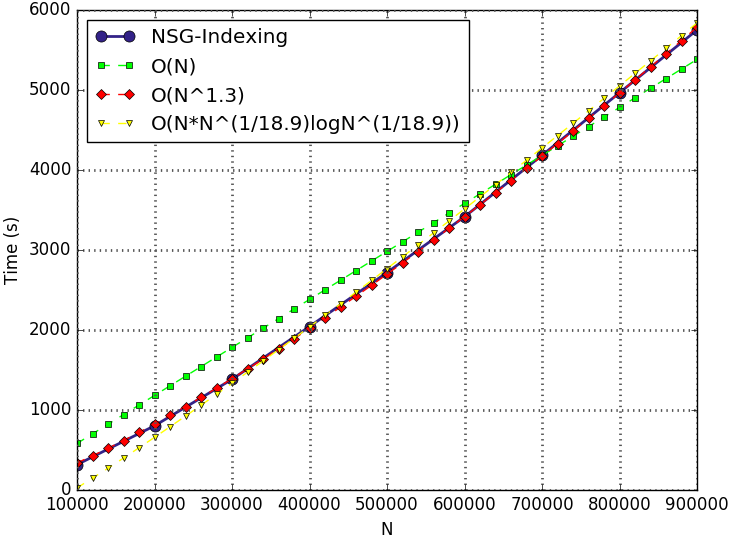}}
	\caption{Indexing time complexity estimation of the NSG on SIFT1M and GIST1M datasets. $N$ is the amount of the data points. The reported time is the processing time of the later steps of the NSG, including the ``search-collect-select'' procedure and the tree spanning procedure. The indexing complexity of NSG is about $O(N^{1+\frac{1}{9.3}}\log N^{\frac{1}{9.3}})$ on SIFT1M, and is close to $O(N^{1.3})$. 9.3 is close to the intrinsic dimension of SIFT1M. The indexing complexity of NSG is about $O(N^{1+\frac{1}{18.9}}\log N^{\frac{1}{18.9}})$ on GIST1M, and is also close to $O(N^{1.3})$. 18.9 is close to the intrinsic dimension of GIST1M.}
	\label{index_complexity}
\end{figure*}

\end{appendix}

\end{document}